\documentclass[11pt]{article}

\usepackage[final]{acl}

\usepackage{times}
\usepackage{latexsym}

\usepackage[T1]{fontenc}

\usepackage[utf8]{inputenc}

\usepackage{microtype}

\usepackage{inconsolata}

\usepackage{bm}
\usepackage{amsmath, amssymb, amsfonts, amsthm}
\usepackage{graphicx}
\usepackage{booktabs}
\usepackage{multirow}
\usepackage{amsfonts}
\usepackage{color,xcolor,colortbl}
\usepackage{adjustbox}
\usepackage{nicematrix}
\usepackage{tabularx}
\usepackage{subcaption}
\usepackage{cleveref}
\usepackage{makecell}
\usepackage[symbol]{footmisc}
\usepackage[inline, shortlabels]{enumitem}

\newtheorem{prop}{Proposition}

\title{Learning Uncertainty from Sequential Internal Dispersion \\in Large Language Models}

\author{
    Ponhvoan Srey$^1$ \enskip
    Xiaobao Wu$^2$\thanks{Corresponding Authors.} \enskip
    Cong-Duy Nguyen$^3$ \enskip
    Anh Tuan Luu$^{1,3}$\footnotemark[1]\\
    $^1$Nanyang Technological University \quad
    $^2$Shanghai Jiao Tong University\\
    $^3$Centre for AI Research, VinUniversity\\
    \texttt{\{ponhvoan002, anhtuan.luu\}@ntu.edu.sg}\\
    \texttt{duy.ntc@vinuni.edu.vn} \quad  \texttt{xiaobaowu@sjtu.edu.cn}
}

\begin{document}
\maketitle
\begin{abstract}

Uncertainty estimation is a promising approach to detect hallucinations in large language models (LLMs).
Recent approaches commonly depend on model internal states to estimate uncertainty.
However, they suffer from strict assumptions on how hidden states should evolve across layers, and from information loss by solely focusing on last or mean tokens.
To address these issues, we present Sequential Internal Variance Representation (SIVR), a supervised hallucination detection framework that leverages token-wise, layer-wise features derived from hidden states.
SIVR adopts a more basic assumption that uncertainty manifests in the degree of dispersion or variance of internal representations across layers, rather than relying on specific assumptions, which makes the method model and task agnostic.
It additionally aggregates the full sequence of per-token variance features, learning temporal patterns indicative of factual errors and thereby preventing information loss.
Experimental results demonstrate SIVR consistently outperforms strong baselines.
Most importantly, SIVR enjoys stronger generalisation and avoids relying on large training sets, highlighting the potential for practical deployment.\footnote[2]{Our code repository is available online at \url{https://github.com/ponhvoan/internal-variance}.}

\end{abstract}

\section{Introduction}

\begin{figure*}[!ht]
    \centering
    \includegraphics[width=\linewidth]{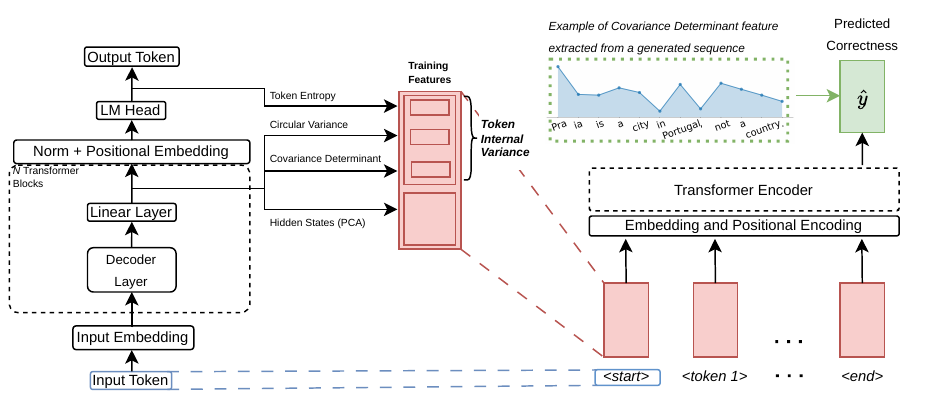}
    \caption{Illustration of our SIVR.
    At each generated token, we extract LLM hidden states, and compute their internal variance, consisting of token entropy, circular variance, and covariance determinant. We utilise these as informative features for sequence classification of response correctness with a simple transformer encoder architecture.}
    \label{fig:reprDiv}
\end{figure*}

Large Language Models (LLMs) have demonstrated impressive growth in performance across a wide array of tasks \citep{achiam2023gpt, grattafiori2024llama}. Increasingly, they are being deployed in complex applications that require sophisticated reasoning, such as coding and mathematical reasoning \citep{guo2025deepseek}. Despite widespread adoption, LLMs invariably suffer from unreliable generation or hallucination, frequently providing fictitious answers with complete confidence \citep{zhang2023siren}. Due to their convincing response, it is challenging for users to determine factual correctness. This poses a major setback to LLM deployment, especially in high-risk domains.

Uncertainty estimation has emerged as a prominent solution to identify incorrect generations by aiming to accurately quantify the level of uncertainty in the response \citep{gal2016dropout, hendrycks2016baseline, lakshminarayanan2017simple}. Robust and accurate uncertainty estimation enables users to determine the level of trust to place in the LLM response, and intervene as necessary. However, recent approaches remain suboptimal, even compared to simple but effective methods, such as computing the entropy of the predictive output probability distribution \citep{vashurin2025benchmarking}. Second, a large majority relies on stochastic sampling to measure the consistency between a set of answers \citep{kuhn2023semantic, lin2024generatingconfidence, manakul2023selfcheckgpt}. This introduces high computational overhead, rendering such methods impractical in real use cases.

Closely related is model probing, which involves training a lightweight classifier on top of last or mean hidden states to identify falsehoods by extracting world knowledge embedded in their internal states \citep{burns2022discovering, li2023inference, azaria2023internal, ji2024llm, marks2023geometry}. These techniques are computationally efficient and enjoy state-of-the-art performance, but require training data and exhibit limited generalisability.

From current research, we identify two gaps in internal state approaches. First, reliance on task and model-specific heuristics for deriving scores reduces transferability across settings. Second, compression of evidence to last or mean-token summaries ignores important sequence-level patterns.
In this work, we tackle hallucination detection using internal LLM signals with an emphasis on cross-layer dynamics and ask: \emph{Can uncertainty be inferred reliably from hidden states over the full token sequence with minimal assumptions to improve cross-task generalisability?} We introduce \textit{\textbf{Sequential Internal Variance Representation (SIVR)}}, a lightweight, supervised uncertainty estimation framework that
\begin{enumerate*}[(i)]
    \item computes token-wise, layer-wise \textbf{\textit{internal variance}} from hidden states, and
    \item dynamically  learns to aggregate the full sequence to estimate uncertainty.
\end{enumerate*}
SIVR operates on the simple, general premise that uncertainty is reflected in the degree of dispersion of internal representations across layers. To this end, we construct per-token internal variance that captures this dispersion, and learn patterns over the sequence to predict factual inaccuracy. Detailed formulation is presented in \Cref{sec:method}. SIVR is efficient to compute, and retains sequence dynamics that are critical for detecting factual errors.
In summary, our contributions are as follows:

\begin{itemize}[leftmargin=*]
    \item We demonstrate the shortcomings of current works that use hidden states as a proxy signal for uncertainty.
    Internal variance, a novel and more robust feature that tracks the dispersion of hidden states across layers, is proposed.
    \item We introduce a pipeline that assesses all tokens, to fully take advantage of patterns suggestive of factual inaccuracy or hallucination.
    \item Our extensive experiments show that SIVR consistently outperforms strong baselines, and the incorporation of the proposed feature significantly enhances OOD generalisation.
\end{itemize}

\section{Methodology}
\label{sec:method}

In this section, we present our proposed uncertainty estimation process in detail. First, we scrutinise a related method, Chain-of-Embeddings (CoE) \citep{wang2024latent}, that produces an uncertainty score, and show its limitations.
Then, we propose a more general feature and incorporate it into a unified framework that ascertains the uncertainty level in the model response.

\subsection{Current Limitations}

\begin{figure}
    \centering
    \includegraphics[width=\linewidth]{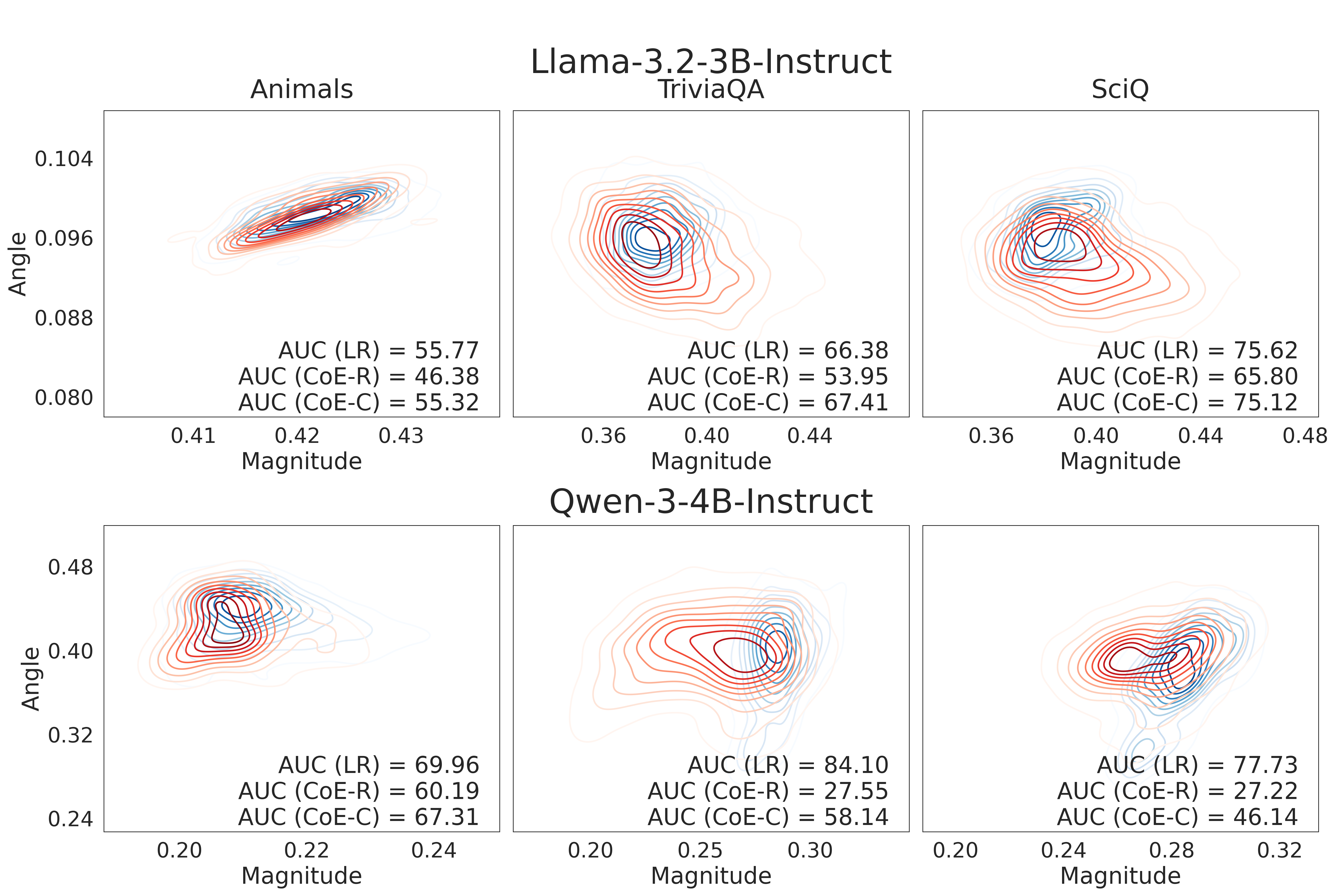}
    \caption{Visualisation of CoE features of \textcolor{blue}{correct} and \textcolor{red}{wrong} answers. Significant overlap indicates CoE features provide poor discriminability.}
    \label{fig:coe_baseline}
\end{figure}

Chain-of-Embeddings (CoE) \citep{wang2024latent} primarily inspects the magnitude and angle of the difference between successive layers of the hidden states. The authors found that answer correctness correlates positively with the magnitude, and negatively with angle, and combined the two features into an uncertainty score based on this observation. This places a strict constraint on the behavior of the hidden states that may not align across models and tasks. In \Cref{fig:coe_baseline}, the CoE features, Magnitude and Angle, of samples from other datasets are reproduced. Between the two models, the pattern of correct statements relative to the wrong ones varies, and there is little distinction between the two classes, resulting in poor discriminability.

Nonetheless, we argue that the hidden states across different layers remain a useful indicator of factual accuracy. For example, \Cref{fig:coe_case} visualises two pairs of generated responses via PCA, where each pair has identical CoE-C scores, but the correct generations show distinctive hidden states features from the wrong ones. This showcases the situation where the score collapses as only the average ``step-size'' between layers is considered even though the latent paths are clearly different.
Furthermore, in \Cref{fig:coe_baseline}, we feed the features to a logistic regression model, and by providing the correctness as labels, a consistently higher AUC is achievable. This confirms that the hidden states, particularly their behavior between layers, are useful features for discerning factual accuracy. Thus, we seek to define a more general and robust feature to encapsulate output uncertainty.

\begin{figure}
    \centering
    \includegraphics[width=\linewidth]{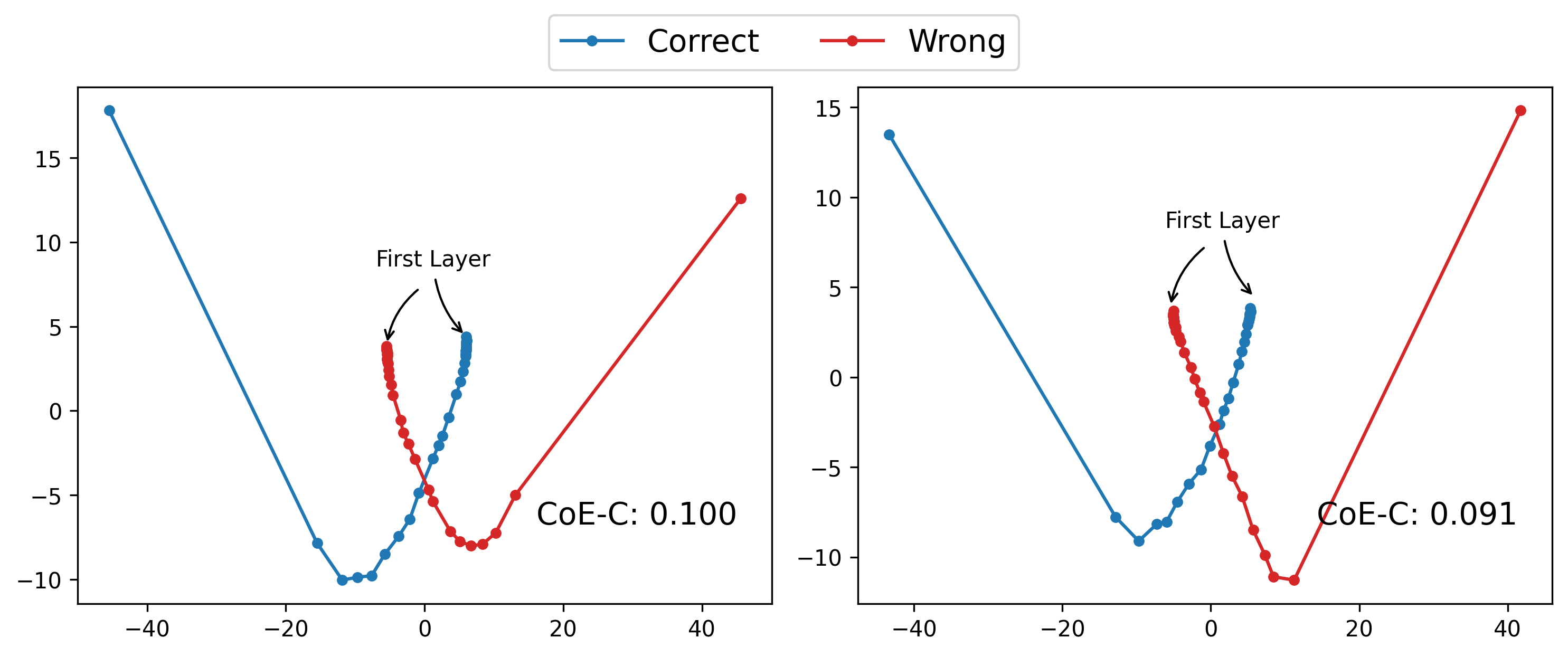}
    \caption{Visualisation (PCA-compressed) of two pairs of embeddings across layers. Each pair has identical CoE-C scores, but distinctive hidden states pattern, showing how CoE-C can degenerate even when the hidden states do not.}
    \label{fig:coe_case}
\end{figure}

\subsection{Internal Variance}

To overcome the above limitations, we aim to formulate an uncertainty score based on a looser, more general assumption: \textit{uncertainty is reflected in the spread of the hidden states across layers}.%

As such, we aim to summarise \emph{cross-layer} variability with compact statistics that capture complementary facets of dispersion and then \emph{learn} patterns over the token sequence.
To characterise latent space geometry, we distinguish between variation in magnitude and direction. Relying on a single metric obscures different degenerate modes. We therefore employ generalised variance to measure volumetric capacity and circular variance for directional diversity.
Appendix~\ref{subsec:gv_cv} shows that generalised variance is bounded by radial and directional components, demonstrating that jointly monitoring both is necessary to isolate distinct failure modes. We complement this geometric analysis by tracking output entropy, linking internal spatial characteristics to the predictive output distribution.

\paragraph{Generalised variance:}
First, we consider the determinant of the covariance matrix or generalised variance of hidden states for all layers.
For the $t$-th output token $s_t$ in a sequence $\bm{s}$ of total length $T$, denote the hidden states or embeddings at the $l$-th layer as $\bm{h}_{t}^{l} \in \mathbb{R}^{d}$, with sample covariance $\Sigma$ where $l \in \{0, \ldots, L\}$ and $d$ is the hidden states dimension.
Since the hidden states are high-dimensional, for more stable covariance estimation, we compute the logarithm of the pseudo-determinant of the regularised covariance $\Sigma' =\Sigma+\alpha I_d$ for some small $\alpha>0$. We approximate the generalised variance as

\begin{equation}
    v_t =\log \det(\Sigma') = \sum_i \log \lambda_i,
\end{equation}

\noindent where $\lambda_i$ are eigenvalues of $\Sigma'$. We use log-determinant $\log\det(\Sigma')$ as a compact summary of multidimensional spread. It aggregates the entire eigen-spectrum and is directly tied to differential entropy in the Gaussian case, an extension of Shannon entropy to continuous space \citep{zhouyin2023understanding}. The log-determinant is computed in a numerically stable and low-cost manner (see Appendix~\ref{sec:compute}).

\paragraph{Circular variance:} Next, we adopt spherical or circular variance on normalised layer vectors $\hat{\bm{h}}_{t}^{l}$, which represents variation in different directions \citep{mardia2009directional}. This provides a complementary and scale-robust view of dispersion. We define circular variance $c_t$ for token  $s_t$ as

\begin{equation}
    c_t = 1 - \left\| \frac{1}{L+1} \sum_{l=0}^{L} \hat{\bm{h}}_t^{l} \right\| .
\end{equation}

Notably, circular variance also contains information on all pairwise relationships of hidden states at different layers (see Appendix~\ref{subsec:circ_var}), instead of being restricted to relationships between successive layers as with CoE. To understand the effectiveness of these new metrics, we compare with CoE-R and CoE-C. In \Cref{fig:internal_variance}, similar to CoE which gives a sequence score, circular variance (CirCVar) and covariance determinant (CovDet) is computed with the mean hidden states $\bar{\bm{h}}^l =\frac{1}{T}\sum_{t}^T \bm{h}_{t}^l$. CovDet and CircVar, on average, yields better AUC (Area Under the receiver operating Curve) than CoE-R and CoE-C. CoE-C generally is quite strong, but in some cases, such as with Qwen on SciQ, it severely underperforms. Meanwhile, the performance of CovDet and CircVar remains robust, indicating their broad applicability.

\begin{figure}
    \centering
    \includegraphics[width=\linewidth]{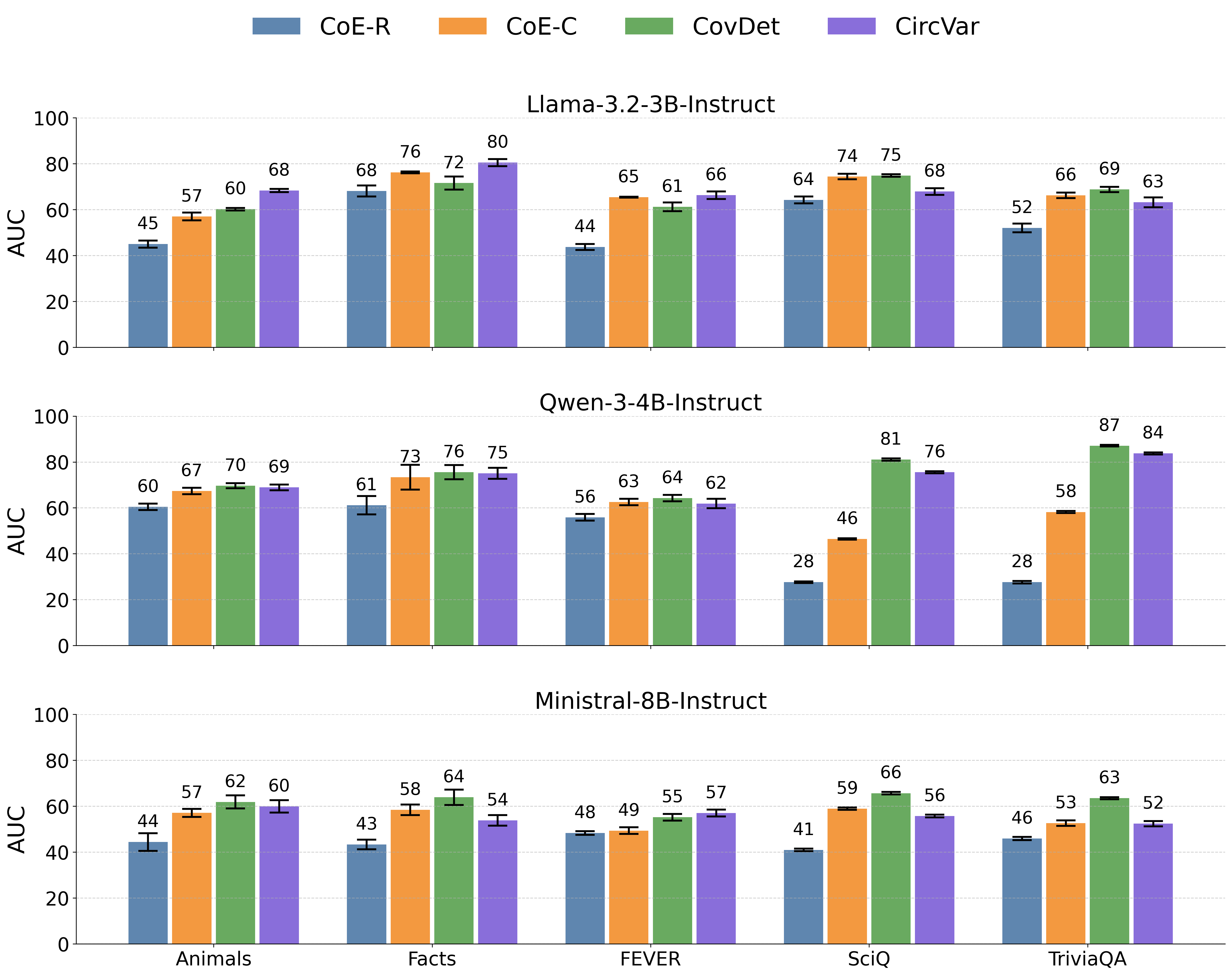}
    \caption{Performance of proposed features compared to CoE.}
    \label{fig:internal_variance}
\end{figure}

\paragraph{Token entropy:}
Finally, we include predictive entropy to quantify uncertainty in the probability distribution emitted by the decoder, which also encodes relevant information on output uncertainty.
We define per-token entropy

\begin{equation}
    e_t = H\left(p(s_t)\right) = - \sum_{s_t \in \mathcal{V}} p (s_t) \log p(s_t)
\end{equation}
\noindent where $p(s_t)$ is the softmax probability of token $s_t \in \mathcal{V}$, and $\mathcal{V}$ is the model vocabulary set.
Collectively, we formulate the \textit{\textbf{internal variance}} at each token as $\bm{v}_t = [v_t, c_t, e_t]^\top$.

Together, these three signals --- generalised variance for magnitude, circular variance for directionality, and predictive entropy for output-space uncertainty --- offer a holistic characterisation of uncertainty that tracks how internal dispersion evolves across layers and how it manifests in the output distribution. Our ablation studies in \Cref{tab:ablation} show that each component contributes and the combination yields the strongest performance.

\subsection{Token Aggregation}

We define a per-token dispersion feature $\bm{v}_t$ and seek a sequence-level uncertainty score. Relying on a single token can miss informative temporal structure. For example (\Cref{fig:reprDiv}), in \texttt{Praia is a city in Portugal, not a country}, the generalised variance is initially high then stabilises, but a sharp spike at \texttt{Portugal} flags the error, a pattern obscured by last or mean-token summaries.
Therefore, we propose to learn from the full token-wise sequence of dispersion features.
We choose to focus on supervised uncertainty learning as this paradigm produces uncertainty scores that are aligned explicitly with a notion of correctness \citep{liu2024uncertainty, srey2026towards}.
Unlike UHead that mean-pool token encodings \citep{shelmanov2025head} for a claim span, we retain order to capture patterns indicative of factual correctness.

Formally, we have responses of variable length $\mathcal{D}=\{\bm{s}_i,y_i\}_{i=1}^{n}$ where $\bm{s}_i\in \mathbb{R}^{T_i\times d}$ and $y_i\in \{0,1\}$ with 1 for hallucinated instances.
From each sequence $\bm{s}_i$, informative training features $\bm{x}_i$, with core components comprising the proposed internal variance $\bm{v}_t$, are extracted. 
The goal is to learn learn $f_\theta : \mathbb{R}^{T\times d_{\text{tr}}} \rightarrow [0,1]$ to estimate $p_\theta (y=1|\bm{x})$, which can be regarded as the uncertainty score for the sequence $\bm{s}$. $d_{\text{tr}}$ is the dimension of the training features. The objective function is the binary cross entropy loss with $l_2$ regularisation

\begin{equation}
    \frac{1}{n}\sum_{i=1}^n -y_i\log\hat{y_i} -(1-y_i)\log{(1-\hat{y_i})} + \beta \|\theta\|_2^2
\end{equation}

We employ a lightweight sequence head with an embedding layer to project input to a hidden dimension of 128, followed by a transformer encoder block, and finally, a linear classifier layer to produce the sequence score. More details on architecture and hyperparameters are provided in Appendix \ref{sec:expt_detail}.

\begin{table*}[t!] 
\renewcommand{\arraystretch}{1.5}
\centering
\setlength{\tabcolsep}{5pt}

\begin{subtable}{\textwidth}
\resizebox{\textwidth}{!}{
\begin{tabular}{l||ccc|ccc|ccc|ccc|ccc|ccc|ccc|ccc|c}
\toprule
\multirow{2}{*}{\textbf{Method}} & \multicolumn{3}{c|}{\textbf{TriviaQA}}& \multicolumn{3}{c|}{\textbf{SciQ}} & \multicolumn{3}{c|}{\textbf{MedMCQA}} & \multicolumn{3}{c|}{\textbf{MGSM}} & \multicolumn{3}{c|}{\textbf{MATH}} & \multicolumn{3}{c|}{\textbf{MMLU}} & \multicolumn{3}{c|}{\textbf{CommonsenseQA}} & \multicolumn{3}{c}{\textbf{Average}} & \multicolumn{1}{|c}{\multirow{2}{*}{\makecell{\textbf{Mean}\\ \textbf{Rank}}}} \\ \cmidrule{2-25}
& \textbf{AUC} & \textbf{F95} & \textbf{AUP} & \textbf{AUC} & \textbf{F95} & \textbf{AUP} & \textbf{AUC} & \textbf{F95} & \textbf{AUP} & \textbf{AUC} & \textbf{F95} & \textbf{AUP} & \textbf{AUC} & \textbf{F95} & \textbf{AUP} & \textbf{AUC} & \textbf{F95} & \textbf{AUP} & \textbf{AUC} & \textbf{F95} & \textbf{AUP} & \textbf{AUC} & \textbf{F95} & \textbf{AUP} \\
\bottomrule
\end{tabular}
}
\caption{Llama-3.2-3B-Instruct}
\end{subtable}

\begin{subtable}{\textwidth}
\resizebox{\textwidth}{!}{
\begin{tabular}{l||ccc|ccc|ccc|ccc|ccc|ccc|ccc|ccc|ccc|ccc|ccc|c}
\toprule
\multirow{2}{*}{\textbf{Method}} & \multicolumn{3}{c|}{\textbf{TriviaQA}}& \multicolumn{3}{c|}{\textbf{SciQ}} & \multicolumn{3}{c|}{\textbf{MedMCQA}} & \multicolumn{3}{c|}{\textbf{MGSM}} & \multicolumn{3}{c|}{\textbf{MATH}} & \multicolumn{3}{c|}{\textbf{MMLU}} & \multicolumn{3}{c|}{\textbf{CommonsenseQA}} & \multicolumn{3}{c}{\textbf{Average}} & \multicolumn{1}{|c}{\multirow{2}{*}{\makecell{\textbf{Mean}\\ \textbf{Rank}}}} \\ \cmidrule{2-25}
& \textbf{AUC} & \textbf{F95} & \textbf{AUP} & \textbf{AUC} & \textbf{F95} & \textbf{AUP} & \textbf{AUC} & \textbf{F95} & \textbf{AUP} & \textbf{AUC} & \textbf{F95} & \textbf{AUP} & \textbf{AUC} & \textbf{F95} & \textbf{AUP} & \textbf{AUC} & \textbf{F95} & \textbf{AUP} & \textbf{AUC} & \textbf{F95} & \textbf{AUP} & \textbf{AUC} & \textbf{F95} & \textbf{AUP} \\
\bottomrule
\end{tabular}
}
\caption{Llama-3.1-8B-Instruct}
\end{subtable}

\begin{subtable}{\textwidth}
\resizebox{\textwidth}{!}{
\begin{tabular}{l||ccc|ccc|ccc|ccc|ccc|ccc|ccc|ccc|c}
\toprule
\multirow{2}{*}{\textbf{Method}} & \multicolumn{3}{c|}{\textbf{TriviaQA}}& \multicolumn{3}{c|}{\textbf{SciQ}} & \multicolumn{3}{c|}{\textbf{MedMCQA}} & \multicolumn{3}{c|}{\textbf{MGSM}} & \multicolumn{3}{c|}{\textbf{MATH}} & \multicolumn{3}{c|}{\textbf{MMLU}} & \multicolumn{3}{c|}{\textbf{CommonsenseQA}} & \multicolumn{3}{c}{\textbf{Average}} & \multicolumn{1}{|c}{\multirow{2}{*}{\makecell{\textbf{Mean}\\ \textbf{Rank}}}} \\ \cmidrule{2-25}
& \textbf{AUC} & \textbf{F95} & \textbf{AUP} & \textbf{AUC} & \textbf{F95} & \textbf{AUP} & \textbf{AUC} & \textbf{F95} & \textbf{AUP} & \textbf{AUC} & \textbf{F95} & \textbf{AUP} & \textbf{AUC} & \textbf{F95} & \textbf{AUP} & \textbf{AUC} & \textbf{F95} & \textbf{AUP} & \textbf{AUC} & \textbf{F95} & \textbf{AUP} & \textbf{AUC} & \textbf{F95} & \textbf{AUP} \\
\bottomrule
\end{tabular}
}
\caption{Ministral-8B-Instruct}
\end{subtable}

\caption{Main results. Cooler color indicates better results, with the best bolded and second best underlined.}
\label{tab:main_result}
\end{table*}

\section{Experiments}

\subsection{Setup}

\paragraph{Datasets.} We evaluate extensively on twelve datasets for two general tasks, fact-checking and question answering (QA). In fact-checking, a statement is provided for an LLM to verify its factuality. The setting is identical to TruthEval \citep{khatun2024trutheval}, and the anchor experiments by \citet{jin2024exploring}. For this task, we select statements from a broad range of topics from the datasets Counterfact \citep{meng2022locating}, Common Claims \citep{casper2023explore}, Animals and Facts \citep{azaria2023internal}, and FEVER \citep{thorne2018fever}.
For clearer presentation, we report performance for fact-checking tasks in Appendix~\ref{subsec:topic}.
In the ablation, we refer to the collection of the first four datasets as ``True-False''.
For QA, we evaluate both free-form, and multiple-choice answers, across multiple important domains: TriviaQA \citep{joshi2017triviaqa}, SciQ \citep{welbl2017crowdsourcing} for general knowledge; MedMCQA \citep{pal2022medmcqa}, MMLU \citep{hendrycks2020measuring} for more specialised medical and scientific knowledge; MGSM \citep{shi2022language}, MATH \citep{hendrycks2021measuring}, CommonsenseQA \citep{talmor2018commonsenseqa} for mathematical and reasoning tasks. Queries from each dataset are provided as prompts for an LLM, and the responses are obtained for evaluation. Corresponding binary correctness labels are extracted via exact matching with reference answers \textit{e.g.} MedMCQA and MMLU, or with ROUGE-L \citep{lin2004rouge} \textit{e.g.} TriviaQA and SciQ. More details are found in Appendix~\ref{sec:expt_detail}.

\paragraph{Baselines.} We compare with four classes of uncertainty estimation baselines:
(1) \textbf{Logit-based}: Maximum Sequence Probability (MSP or Max Prob), Entropy \citep{huang2023look}, Perplexity \citep{si2022prompting}, Temperature-Scaled MSP \citep{shih2023long}, and Energy \citep{liu2020energy}.
(2) \textbf{Sampling-based}: Semantic Entropy (SE) \citep{kuhn2023semantic}, SAR \citep{duan-etal-2024-shifting}.
(3) \textbf{Confidence elicitation}: P(True) \citep{kadavath2022language}, which prompts the model to determine if the answer is correct, and the probability of the token ``True'' is extracted as a confidence score.
(4) \textbf{Internal states:} CoE \citep{wang2024latent}; and four supervised baselines:
SAPLMA \citep{azaria2023internal}: trained on last-token hidden states;
SATMD + MSP \citep{vazhentsev2025token}: trained on Mahalanobis distances of hidden states at each layer, averaged across tokens, and concatenated with MSP;
Lookback Lens \citep{chuang-etal-2024-lookback}: trained on lookback ratios, the proportion of attention on context to attention on both generated and context tokens; and
TAD: a two-stage process that trains a probe on attention weights, token probabilities, and intermediate scores from the first stage \citep{vazhentsev-etal-2025-unconditional}.

\paragraph{Language Models.}
For the main experiments, we evaluate with Llama-3.2-3B, Llama-3.1-8B \citep{grattafiori2024llama}, and Ministral-8B \citep{Mistral8BInstruct}.
In Appendix~\ref{subsec:qwen_result}, we provide supplementary results with Qwen-3-4B and Qwen-3-14B \citep{yang2025qwen3}.
We work with the instruction-tuned versions, except for Qwen-3-14B.
For reproducibility, we generate all responses via greedy decoding.

\paragraph{Metrics.} We employ AUC, FPR@95, and AUPR \citep{davis2006relationship} as evaluation metrics to comprehensively assess the quality of the uncertainty scores for discerning correct from wrong generations. AUC (Area Under the receiver operating Curve) is a threshold-invariant metric that measures the classifier's ability to distinguish between positive and negative classes. A higher AUC indicates stronger discriminative power. FPR@95 quantifies the proportion of negatives incorrectly classified as positives when a true positive rate or recall of 95\% is achieved, a particularly informative metric for safety-critical contexts with stringent recall requirements. A low FPR@95 is better.
Finally, AUPR (Area Under the Precision-Recall curve) is a threshold-invariant metric that is concerned with the trade-off between precision and recall. A higher AUPR reflects the classifier's ability to maintain precision while retrieving positives effectively.

\subsection{Results and Analysis}
\label{subsec:results}

\begin{table*}[t!] 
\renewcommand{\arraystretch}{1.35}
\centering
\setlength{\tabcolsep}{1pt}

\resizebox{\textwidth}{!}{
\begin{tabular}{l||ccc|ccc|ccc|ccc|ccc|ccc|ccc|ccc|ccc|ccc}
\toprule
\multirow{2}{*}{\textbf{Variant}} & \multicolumn{3}{c|}{\textbf{True-False}} & \multicolumn{3}{c|}{\textbf{Fever}} & \multicolumn{3}{c|}{\textbf{TriviaQA}}& \multicolumn{3}{c|}{\textbf{SciQ}} & \multicolumn{3}{c|}{\textbf{MedMCQA}} & \multicolumn{3}{c|}{\textbf{GSM}} & \multicolumn{3}{c|}{\textbf{Math}} & \multicolumn{3}{c|}{\textbf{MMLU}} &  \multicolumn{3}{c|}{\textbf{CommonsenseQA}} & \multicolumn{3}{c}{\textbf{Average}} \\ \cline{2-31}
    & \textbf{AUC} & \textbf{F95} & \textbf{AUP} & \textbf{AUC} & \textbf{F95} & \textbf{AUP} & \textbf{AUC} & \textbf{F95} & \textbf{AUP} & \textbf{AUC} & \textbf{F95} & \textbf{AUP} & \textbf{AUC} & \textbf{F95} & \textbf{AUP} & \textbf{AUC} & \textbf{F95} & \textbf{AUP} & \textbf{AUC} & \textbf{F95} & \textbf{AUP} & \textbf{AUC} & \textbf{F95} & \textbf{AUP} & \textbf{AUC} & \textbf{F95} & \textbf{AUP} & \textbf{AUC} & \textbf{F95} & \textbf{AUP} \\
\midrule
HS (PCA) & \large\cellcolor[rgb]{0.98662,0.88466,0.82468} 79.33 & \large\cellcolor[rgb]{0.97416,0.75087,0.64245} 72.22 & \large\cellcolor[rgb]{0.90142,0.93679,0.95625} 53.13 & \large\cellcolor[rgb]{0.78085,0.87835,0.9301} 80.78 & \large\cellcolor[rgb]{0.69366,0.83406,0.90519} 49.69 & \large\cellcolor[rgb]{0.84883,0.91188,0.94656} 50.83 & \large\cellcolor[rgb]{0.80992,0.89312,0.93841} 84.65 & \large\cellcolor[rgb]{0.98016,0.9148,0.87636} 61.54 & \large\cellcolor[rgb]{0.81961,0.89804,0.94118} 84.23 & \large\cellcolor[rgb]{0.99077,0.85052,0.76978} 76.31 & \large\cellcolor[rgb]{0.97416,0.75087,0.64245} 76.74 & \large\cellcolor[rgb]{0.97001,0.96217,0.95755} 81.93 & \large\cellcolor[rgb]{0.97278,0.94925,0.93541} 59.66 & \large\cellcolor[rgb]{0.84298,0.90911,0.94548} 83.96 & \large\cellcolor[rgb]{0.70334,0.83899,0.90796} 62.94 & \large\cellcolor[rgb]{0.86636,0.92018,0.94979} 80.84 & \large\cellcolor[rgb]{0.86636,0.92018,0.94979} 54.20 & \large\cellcolor[rgb]{0.8722,0.92295,0.95087} 67.80 & \large\cellcolor[rgb]{0.90142,0.93679,0.95625} 77.23 & \large\cellcolor[rgb]{0.93064,0.95063,0.96163} 68.04 & \large\cellcolor[rgb]{0.84883,0.91188,0.94656} 78.54 & \large\cellcolor[rgb]{0.97093,0.95786,0.95017} 56.33 & \large\cellcolor[rgb]{0.69366,0.83406,0.90519} 85.37 & \large\cellcolor[rgb]{0.9248,0.94787,0.96055} 77.65 & \large\cellcolor[rgb]{0.97186,0.95356,0.94279} 63.08 & \large\cellcolor[rgb]{0.69366,0.83406,0.90519} 83.54 & \large\cellcolor[rgb]{0.99077,0.85052,0.76978} 45.43 & \large\cellcolor[rgb]{0.96909,0.96647,0.96494} 74.68 & \large\cellcolor[rgb]{0.98662,0.82561,0.73795} 71.00 & \large\cellcolor[rgb]{0.90727,0.93956,0.95732} 63.49 \\ 
HS & \large\cellcolor[rgb]{0.87805,0.92572,0.95194} 82.98 & \large\cellcolor[rgb]{0.76148,0.86851,0.92457} 47.48 & \large\cellcolor[rgb]{0.91311,0.94233,0.9584} 52.86 & \large\cellcolor[rgb]{0.86636,0.92018,0.94979} 78.54 & \large\cellcolor[rgb]{0.83714,0.90634,0.94441} 55.35 & \large\cellcolor[rgb]{0.9248,0.94787,0.96055} 47.43 & \large\cellcolor[rgb]{0.88389,0.92849,0.95302} 81.77 & \large\cellcolor[rgb]{0.98201,0.90619,0.86159} 62.50 & \large\cellcolor[rgb]{0.93649,0.9534,0.96271} 79.73 & \large\cellcolor[rgb]{0.97416,0.75087,0.64245} 75.29 & \large\cellcolor[rgb]{0.97463,0.94064,0.92065} 68.60 & \large\cellcolor[rgb]{0.97416,0.75087,0.64245} 77.98 & \large\cellcolor[rgb]{0.97416,0.75087,0.64245} 55.89 & \large\cellcolor[rgb]{0.97647,0.93203,0.90588} 90.57 & \large\cellcolor[rgb]{0.97416,0.75087,0.64245} 53.13 & \large\cellcolor[rgb]{0.80992,0.89312,0.93841} 82.49 & \large\cellcolor[rgb]{0.80023,0.8882,0.93564} 50.38 & \large\cellcolor[rgb]{0.84883,0.91188,0.94656} 68.71 & \large\cellcolor[rgb]{0.69366,0.83406,0.90519} 80.01 & \large\cellcolor[rgb]{0.69366,0.83406,0.90519} 62.89 & \large\cellcolor[rgb]{0.80992,0.89312,0.93841} 79.21 & \large\cellcolor[rgb]{0.97416,0.75087,0.64245} 55.00 & \large\cellcolor[rgb]{0.99031,0.86744,0.79516} 90.24 & \large\cellcolor[rgb]{0.97416,0.75087,0.64245} 74.40 & \large\cellcolor[rgb]{0.87805,0.92572,0.95194} 64.58 & \large\cellcolor[rgb]{0.97416,0.75087,0.64245} 95.12 & \large\cellcolor[rgb]{0.69366,0.83406,0.90519} 47.22 & \large\cellcolor[rgb]{0.91895,0.9451,0.95948} 75.46 & \large\cellcolor[rgb]{0.86052,0.91742,0.94871} 63.80 & \large\cellcolor[rgb]{0.97739,0.92772,0.8985} 61.60 \\ 
CovDet, AngVar & \large\cellcolor[rgb]{0.97416,0.75087,0.64245} 77.38 & \large\cellcolor[rgb]{0.97278,0.94925,0.93541} 60.36 & \large\cellcolor[rgb]{0.88973,0.93126,0.95409} 53.43 & \large\cellcolor[rgb]{0.97001,0.96217,0.95755} 75.04 & \large\cellcolor[rgb]{0.97832,0.92341,0.89112} 66.67 & \large\cellcolor[rgb]{0.99031,0.86744,0.79516} 38.91 & \large\cellcolor[rgb]{0.97416,0.75087,0.64245} 69.85 & \large\cellcolor[rgb]{0.97416,0.75087,0.64245} 73.08 & \large\cellcolor[rgb]{0.97416,0.75087,0.64245} 70.00 & \large\cellcolor[rgb]{0.99031,0.86744,0.79516} 76.58 & \large\cellcolor[rgb]{0.84883,0.91188,0.94656} 61.63 & \large\cellcolor[rgb]{0.99031,0.86744,0.79516} 79.54 & \large\cellcolor[rgb]{0.79054,0.88328,0.93287} 63.41 & \large\cellcolor[rgb]{0.97416,0.75087,0.64245} 97.17 & \large\cellcolor[rgb]{0.85467,0.91465,0.94764} 60.67 & \large\cellcolor[rgb]{0.97416,0.75087,0.64245} 70.28 & \large\cellcolor[rgb]{0.97416,0.75087,0.64245} 77.10 & \large\cellcolor[rgb]{0.97416,0.75087,0.64245} 54.48 & \large\cellcolor[rgb]{0.98524,0.8173,0.72734} 72.85 & \large\cellcolor[rgb]{0.97416,0.75087,0.64245} 75.26 & \large\cellcolor[rgb]{0.97416,0.75087,0.64245} 71.69 & \large\cellcolor[rgb]{0.69366,0.83406,0.90519} 57.88 & \large\cellcolor[rgb]{0.97555,0.93633,0.91326} 89.02 & \large\cellcolor[rgb]{0.69366,0.83406,0.90519} 79.78 & \large\cellcolor[rgb]{0.97416,0.75087,0.64245} 60.32 & \large\cellcolor[rgb]{0.83714,0.90634,0.94441} 85.98 & \large\cellcolor[rgb]{0.97416,0.75087,0.64245} 45.09 & \large\cellcolor[rgb]{0.97416,0.75087,0.64245} 71.31 & \large\cellcolor[rgb]{0.97416,0.75087,0.64245} 72.28 & \large\cellcolor[rgb]{0.98385,0.809,0.71672} 59.49 \\ 
Internal Variance & \large\cellcolor[rgb]{0.8722,0.92295,0.95087} 83.08 & \large\cellcolor[rgb]{0.94817,0.95894,0.96486} 57.45 & \large\cellcolor[rgb]{0.70334,0.83899,0.90796} 57.34 & \large\cellcolor[rgb]{0.97463,0.94064,0.92065} 73.98 & \large\cellcolor[rgb]{0.97416,0.75087,0.64245} 76.10 & \large\cellcolor[rgb]{0.98108,0.9105,0.86897} 41.60 & \large\cellcolor[rgb]{0.8313,0.90358,0.94333} 83.76 & \large\cellcolor[rgb]{0.69366,0.83406,0.90519} 38.46 & \large\cellcolor[rgb]{0.9248,0.94787,0.96055} 80.08 & \large\cellcolor[rgb]{0.69366,0.83406,0.90519} 81.92 & \large\cellcolor[rgb]{0.69366,0.83406,0.90519} 56.98 & \large\cellcolor[rgb]{0.77116,0.87343,0.92734} 85.32 & \large\cellcolor[rgb]{0.69366,0.83406,0.90519} 64.53 & \large\cellcolor[rgb]{0.9797,0.78408,0.68489} 96.23 & \large\cellcolor[rgb]{0.69366,0.83406,0.90519} 63.10 & \large\cellcolor[rgb]{0.86052,0.91742,0.94871} 80.97 & \large\cellcolor[rgb]{0.90727,0.93956,0.95732} 57.25 & \large\cellcolor[rgb]{0.88389,0.92849,0.95302} 67.21 & \large\cellcolor[rgb]{0.84298,0.90911,0.94548} 78.25 & \large\cellcolor[rgb]{0.98662,0.88466,0.82468} 72.16 & \large\cellcolor[rgb]{0.69366,0.83406,0.90519} 80.57 & \large\cellcolor[rgb]{0.98662,0.88466,0.82468} 55.72 & \large\cellcolor[rgb]{0.97555,0.93633,0.91326} 89.02 & \large\cellcolor[rgb]{0.98754,0.88035,0.8173} 75.70 & \large\cellcolor[rgb]{0.91895,0.9451,0.95948} 64.05 & \large\cellcolor[rgb]{0.88389,0.92849,0.95302} 87.20 & \large\cellcolor[rgb]{0.97186,0.95356,0.94279} 46.06 & \large\cellcolor[rgb]{0.86636,0.92018,0.94979} 76.29 & \large\cellcolor[rgb]{0.97093,0.95786,0.95017} 66.93 & \large\cellcolor[rgb]{0.8722,0.92295,0.95087} 64.08 \\ 
Full (RNN) & \large\cellcolor[rgb]{0.98939,0.84221,0.75917} 78.56 & \large\cellcolor[rgb]{0.97555,0.93633,0.91326} 61.46 & \large\cellcolor[rgb]{0.97416,0.75087,0.64245} 44.97 & \large\cellcolor[rgb]{0.97416,0.75087,0.64245} 68.02 & \large\cellcolor[rgb]{0.99216,0.85882,0.78039} 71.70 & \large\cellcolor[rgb]{0.97416,0.75087,0.64245} 34.93 & \large\cellcolor[rgb]{0.77116,0.87343,0.92734} 85.45 & \large\cellcolor[rgb]{0.89558,0.93403,0.95517} 50.00 & \large\cellcolor[rgb]{0.80023,0.8882,0.93564} 84.71 & \large\cellcolor[rgb]{0.76148,0.86851,0.92457} 81.35 & \large\cellcolor[rgb]{0.98754,0.88035,0.8173} 72.09 & \large\cellcolor[rgb]{0.69366,0.83406,0.90519} 86.19 & \large\cellcolor[rgb]{0.97463,0.94064,0.92065} 59.48 & \large\cellcolor[rgb]{0.69366,0.83406,0.90519} 80.19 & \large\cellcolor[rgb]{0.98201,0.90619,0.86159} 56.29 & \large\cellcolor[rgb]{0.75179,0.86359,0.9218} 83.65 & \large\cellcolor[rgb]{0.69366,0.83406,0.90519} 45.80 & \large\cellcolor[rgb]{0.82545,0.90081,0.94225} 69.79 & \large\cellcolor[rgb]{0.97416,0.75087,0.64245} 72.07 & \large\cellcolor[rgb]{0.98662,0.88466,0.82468} 72.16 & \large\cellcolor[rgb]{0.97416,0.75087,0.64245} 71.70 & \large\cellcolor[rgb]{0.85467,0.91465,0.94764} 57.15 & \large\cellcolor[rgb]{0.97555,0.93633,0.91326} 89.02 & \large\cellcolor[rgb]{0.94233,0.95617,0.96378} 77.37 & \large\cellcolor[rgb]{0.90727,0.93956,0.95732} 64.18 & \large\cellcolor[rgb]{0.86052,0.91742,0.94871} 86.59 & \large\cellcolor[rgb]{0.97693,0.76747,0.66367} 45.14 & \large\cellcolor[rgb]{0.97832,0.92341,0.89112} 73.80 & \large\cellcolor[rgb]{0.97647,0.93203,0.90588} 67.78 & \large\cellcolor[rgb]{0.97416,0.75087,0.64245} 58.83 \\ 
SIVR & \large\cellcolor[rgb]{0.69366,0.83406,0.90519} 85.37 & \large\cellcolor[rgb]{0.69366,0.83406,0.90519} 45.06 & \large\cellcolor[rgb]{0.69366,0.83406,0.90519} 57.54 & \large\cellcolor[rgb]{0.69366,0.83406,0.90519} 82.44 & \large\cellcolor[rgb]{0.81961,0.89804,0.94118} 54.09 & \large\cellcolor[rgb]{0.69366,0.83406,0.90519} 55.71 & \large\cellcolor[rgb]{0.69366,0.83406,0.90519} 87.29 & \large\cellcolor[rgb]{0.73241,0.85375,0.91626} 40.38 & \large\cellcolor[rgb]{0.69366,0.83406,0.90519} 87.07 & \large\cellcolor[rgb]{0.81961,0.89804,0.94118} 80.81 & \large\cellcolor[rgb]{0.90142,0.93679,0.95625} 63.95 & \large\cellcolor[rgb]{0.70334,0.83899,0.90796} 86.03 & \large\cellcolor[rgb]{0.9248,0.94787,0.96055} 61.12 & \large\cellcolor[rgb]{0.78085,0.87835,0.9301} 82.08 & \large\cellcolor[rgb]{0.84883,0.91188,0.94656} 60.76 & \large\cellcolor[rgb]{0.69366,0.83406,0.90519} 84.70 & \large\cellcolor[rgb]{0.93064,0.95063,0.96163} 58.78 & \large\cellcolor[rgb]{0.69366,0.83406,0.90519} 73.12 & \large\cellcolor[rgb]{0.88389,0.92849,0.95302} 77.56 & \large\cellcolor[rgb]{0.98662,0.88466,0.82468} 72.16 & \large\cellcolor[rgb]{0.91311,0.94233,0.9584} 77.19 & \large\cellcolor[rgb]{0.99123,0.86313,0.78777} 55.54 & \large\cellcolor[rgb]{0.97416,0.75087,0.64245} 91.46 & \large\cellcolor[rgb]{0.91895,0.9451,0.95948} 77.69 & \large\cellcolor[rgb]{0.69366,0.83406,0.90519} 66.39 & \large\cellcolor[rgb]{0.86052,0.91742,0.94871} 86.59 & \large\cellcolor[rgb]{0.91895,0.9451,0.95948} 46.40 & \large\cellcolor[rgb]{0.69366,0.83406,0.90519} 78.11 & \large\cellcolor[rgb]{0.69366,0.83406,0.90519} 60.81 & \large\cellcolor[rgb]{0.69366,0.83406,0.90519} 66.18 \\
\bottomrule
\end{tabular}
}

\caption{Ablation with Ministral-8B-Instruct}
\label{tab:ablation}

\end{table*}

\paragraph{Main Results.} \Cref{tab:main_result} presents the main results. First, our approach consistently achieves better results compared to its competitors. On average, across all datasets and models, the strongest baseline, TAD, obtains 76.44\%, 64.74\%, and 65.01\% on the three metrics. In comparison, SIVR achieves 79.50\%, 57.21\%, and 69.20\%, amounting to an improvement of 3.06\%, 7.53\%, and 4.19\% on AUC, FPR@95, and AUPR, respectively.
Similar to previous findings \citep{fadeeva2023lmpolygraph, vashurin2025benchmarking}, simple methods such as Max Prob are relatively effective and stable.
SAR and SE slightly outperform these methods, but are not consistent. This may be attributed to the difficulty sampling-based approaches face in detecting self-consistent errors \citep{tan-etal-2025-consistent}.
Likewise, most other methods lack stability.
P(True) generally gives weak results, but on certain tasks, such as Common Claim, the performance is fair. This may be due to sensitivity of input prompts and the token selected for confidence elicitation. 
In some cases, CoE is competitive even with supervised baselines, such as on SciQ and MedMCQA (Llama-3.2-3B-Instruct), but performance varies significantly across tasks.
This reinforces the observation that CoE could provide SOTA results, but typically underperforms because their assumption does not hold in general. 
In line with prior works, supervised methods are the most optimal and stable. They are comparable on most tasks, but on the whole, SAPLMA surpasses both SATMD and Lookback Lens, suggesting that hidden states themselves are more informative than Mahalanobis distances and lookback ratios. Meanwhile, our approach additionally inspects internal variance features sequentially, enabling us to more accurately determine the likelihood of incorrect or hallucinated generations. We further show that this inclusion of internal variance is crucial to maintain performance in out-of-distribution settings.

\paragraph{Ablation Study.} To understand the effectiveness of the components in SIVR, we conduct an ablation. Two settings are considered: one only using the hidden states, and one only the internal variance. For the first setting, we test with hidden states, and with top-10 principal components (PC) to reduce overfitting. We conduct a more extensive ablation on the number of PCs in Appendix~\ref{subsec:pc_ablation}. Then, we test only the proposed CovDet and CircVar, as well as the whole internal variance, which includes token entropy. Finally, we employ a recurrent neural network (RNN) as the architecture. The results are provided in Table \ref{tab:ablation}. Using only the hidden states slightly outperforms SAPLMA. As using only the PC hidden states yield reasonable performance, these are adopted into the complete SIVR. On the other hand, with just CovDet and CircVar, we achieve an AUC of 71.31\%, only slightly worse than SAPLMA. Further inclusion of token entropy and hidden states boost performance considerably.
Notably, even though internal variance is only 3-dimensional compared to the much higher dimensional hidden states, using internal variance as features alone yields comparable results to using the full hidden states. This demonstrates that internal variance is informative as training features. The combination of both internal variance and hidden states in the full SIVR framework gives the optimal results.
As the RNN is not as effective as the transformer encoder, transformer is used in the main method.

\begin{figure*}[ht]
    \centering
    \includegraphics[width=\linewidth]{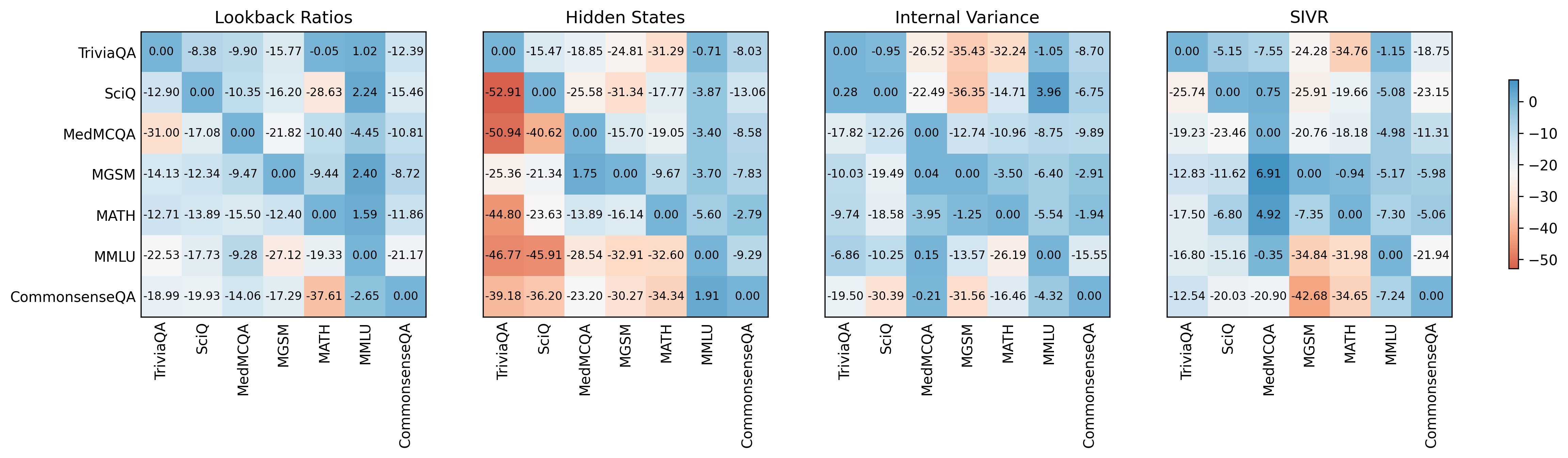}
    \caption{Change in AUC under OOD setting with Ministral-8B-Instruct. Training and test data are on vertical and horizontal axis, respectively.}
    \label{fig:ood_auc}
\end{figure*}

\begin{table}[ht]
    \centering
    \setlength{\tabcolsep}{1.15pt}
    \renewcommand{\arraystretch}{1.25}
    \resizebox{\linewidth}{!}{
    \begin{tabular}{l| c c c c c c c c | c}
         \toprule
         \textbf{Test On} & Trivia & SciQ & MCQA & MGSM & MATH & MMLU & CSQA & Avg & Test AUC\\
         \midrule
         \textbf{Lookback Ratios} & -18.71 & -14.89 & -11.43 & -18.43 & -17.58 & 0.02 & -13.4 & -13.49 & 55.74 \\
         \textbf{Hidden States} & -43.32 & -30.53 & -18.05 & -25.19 & -24.12 & -2.56 & -8.26 & -21.72 & 49.44 \\
         \textbf{Internal Variance} & -10.61 & -15.32 & -8.83 & -21.82 & -17.34 & -3.68 & -7.62 & -12.18 & 60.56 \\
         \textbf{SIVR} & -17.44 & -13.7 & -2.7 & -25.97 & -23.36 & -5.16 & -14.36 & -14.67 & 58.67 \\
         \bottomrule
    \end{tabular}
    }
    \caption{Average change in AUC with out-of-distribution training data.}
    \label{tab:ood_summary}
\end{table}

\paragraph{Out-of-distribution (OOD) Evaluation.} While the supervised paradigm is SOTA, performance usually degrades drastically under the OOD setting where test data differ from the training set. Specifically, to test and compare the robustness of SIVR, we utilise:
\begin{enumerate*}[(i)]
    \item lookback ratios;
    \item only the hidden states;
    \item only the internal variance; and
    \item the full SIVR features.
\end{enumerate*}
Lookback ratios were specifically selected for comparison as they boast strong transfer capability. \Cref{fig:ood_auc} visualises the change in AUC, and \Cref{tab:ood_summary} summarises the result for each configuration. As expected, a considerable decline in AUC is observed in all cases.
Nonetheless, internal variance is the most robust, leading to only a 12.18\% decrease compared to using hidden states which results in a drop of 21.72\%.
OOD AUC with internal variance even exceeds that using lookback ratios by 4.8\%, which were designed to transfer across tasks.
We attribute this to the fact that hidden states are higher dimensional and can encode more task-specific information, making them easier to overfit, whereas internal variance captures more invariant cross-layer dispersion patterns associated with errors, hence better transfer.
This behaviour is also aligned with prior observations that hidden-states-based probes can overfit and degrade substantially when transferred to OOD tasks \citep{chuang-etal-2024-lookback}.
AUPR and FPR@95, provided in Appendix~\ref{subsec:extra_ood}, further corroborate our findings. Overall, this underscores the contribution of internal variance, particularly to improve generalisability.

\paragraph{Training Data Size.} \Cref{fig:data_size} plots the evaluation with the three metrics for varying number of training instances. More training data benefit performance. However, even with only 128 data points, a satisfactory result is achieved, reducing the dependence on training data size.

\begin{figure}[htbp]
    \centering
    \includegraphics[width=\linewidth]{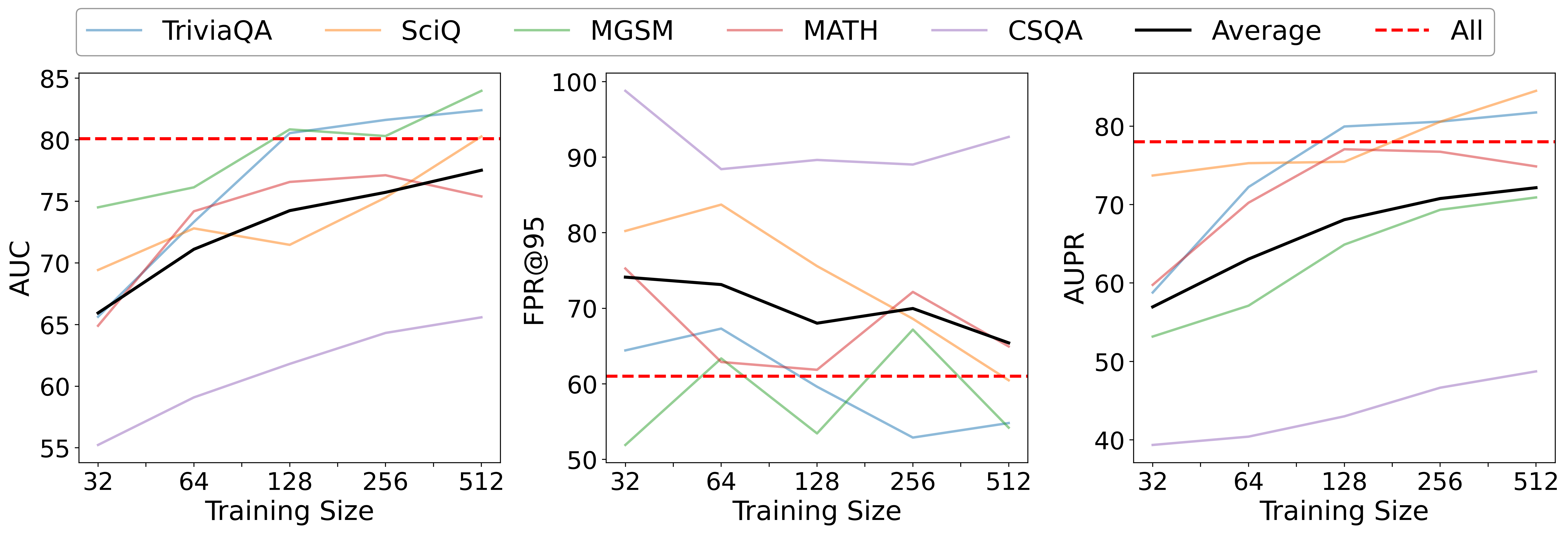}
    \caption{Effect of training data size.}
    \label{fig:data_size}
\end{figure}

\section{Related Work}

Interest in uncertainty estimation for language models has led to a surge in proposals. 
Recent works build on conventional information-based approaches, such as entropy, and propose adaptations specific to language generation by combining logit-level and language-level uncertainty estimation \citep{kuhn2023semantic, duan-etal-2024-shifting, zhang2023enhancing}.
\citet{kuhn2023semantic} noticed that simple information-based methods inappropriately increased uncertainty for syntactically different but semantically identical generations, and introduced semantic uncertainty to address this discrepancy. 
\citet{duan-etal-2024-shifting} proposed SAR to eliminate the effects of prevalent but semantically trivial tokens.
Concerned with a similar issue, Stronger Focus \citep{zhang2023enhancing} estimates entropy, weighted by attention values. \citet{vazhentsev2025uncertainty} discovered that, for hallucinated tokens, attention weights noticeably drop for ``uncertainty-aware'' heads, which are simply those with the highest average attention. Subsequently, confidence scores are computed for these heads and aggregated for the sequence. 

To extract internal knowledge in LLMs, the next suite of techniques formulate an uncertainty score from model behaviour, mainly focusing on the hidden states.
At the broadest level, we categorise these techniques into probe-free and probe-based methods.
Probe-free methods utilise the hidden states to derive an uncertainty score. They can be further divided into
\begin{enumerate*}[(i)]
    \item Sampling-based, such as EigenScore \citep{chen2024inside} which measures consistency in hidden states of a set of sampled responses;
    \item density-based approaches, which take inspiration from out-of-distribution detection \citep{yoo2022detection, ren2022out}, and compute some measure of distance, notably the Mahalanobis distance \citep{lee2018simple} to the data centroid; and
    \item single-call approaches, such as CoE \citep{wang2024latent} which tracks the change in hidden states of successive layers from a single generation as a reflection of predictive uncertainty.
\end{enumerate*}

On the other hand, probe-based methods train a lightweight classifier or \emph{probe} on features extracted from the LLM.
\citet{azaria2023internal} took hidden states at the last layer of the last token directly.
Subsequent works focus on extracting more informative features.
For example, \citet{vazhentsev2025token} computed their Mahalanobis distances of embeddings at each layer, averaged over tokens in the sequence, and concatenated them with the maximum sequence probability. On top of activation maps at each layer, \citet{he2024llm} also considered the top-$k$ probabilities and indices of the last token.
Other works seek to alleviate the burden of label annotation, for instance, through automatic data generation and labelling with high-capacity commercial models \citep{su2024unsupervised, vazhentsev2025uncertainty}, or by utilising alternative signal sources from the model itself, like self-verbalised confidence \citep{srey-etal-2025-unsupervised}, as soft pseudolabels.
However, performance, especially for logits and language-based scores, is normally subpar compared to simpler methods \citep{vashurin2025benchmarking}, and often, they rely on sampling, resulting in high computational costs.
Methods that formulate uncertainty using hidden states, such as CoE \citep{wang2024latent}, have strict assumptions. Meanwhile, their probe-based counterparts are empirically strong, but exhibit limited generalisation.
To fill this gap, SIVR introduces a principled label-efficient, and generalisable sequence-aware signal.

Moreover, current methods mostly output an uncertainty score assessed solely on the last token \citep{he2024llm}, mean token \citep{wang2024latent}, or as a mean of token scores \citep{vazhentsev2025uncertainty}, disregarding relevant tokens and information that encode uncertainty. \citet{shelmanov2025head} passed attention features through a transformer, and averaged encoder outputs across claim tokens before classification. In contrast, SIVR operates without claim segmentation, uses a new family of cross-layer internal-variance features, and preserves token order to 
learn sequence-level dispersion patterns, enabling fine-grained token-level attribution even when claim spans are unavailable (see examples in Appendix~\ref{subsec:example}).

\section{Conclusion}

In this work, we introduced SIVR, a novel factual detection framework that tracks model internal dispersion as it generates a response. Given the hidden states at each token, we calculate their internal variance, a measure of their dispersion between layers. 
Internal variance, along with the hidden states, are extracted as training features for sequence classification. Following our framework, SIVR dynamically learns patterns indicative of factual correctness. SIVR demonstrates strong and stable performance, consistently surpassing competitive baselines. Crucially, we show that internal variance significantly enhances out-of-distribution generalisation. Moreover, a large training dataset is not needed, outlining the potential for real-world applications.
In future work, there are several directions to consider. Other informative features, such as semantic importance, could be incorporated as extra signals for uncertainty. 
More broadly, to reduce reliance on annotated data, unsupervised approaches leveraging internal variance offer a promising direction towards practical solutions for factual detection.
Further integration into the decoding process can enable inference-time hallucination detection and mitigation, and improve model reliability.

\section*{Limitations}

We believe our work has the following limitations.

\paragraph{Supervised Learning.}
While SIVR remains relatively robust to OOD tasks and works well with few data points, it is still supervised and requires an annotated training data.

\paragraph{LLM Variety.} We experimented with moderately sized LLMs, ranging from 3 to 14 billion parameters. Experiments with more capable models of larger size could further validate our claims.

\paragraph{Interpretability.} Our work opens up new avenues for exploring the interpretability of the training features at each token. 
In Appendix~\ref{subsec:example}, we have demonstrated the potential to leverage attribution methods to study the contribution of each token to hallucination risk. Future work could consider token or claim-level evaluation with groundtruths.

\section*{Acknowledgements}

We are grateful to all anonymous reviewers for providing constructive and helpful feedback to strengthen our work. This research is supported by the National Research Foundation, Singapore under its National Large Language Models Funding Initiative (AISG Award No: AISG-NMLP-2024-005).
\bibliography{ref}

\clearpage
\appendix

\section{Theoretical Analysis}

\subsection{Generalised Variance Bound}
\label{subsec:gv_cv}

\Cref{prop:gv_bound} demonstrates that generalised variance is constrained by both sources of magnitude (radial) and directional variation. Using generalised variance alone is insufficient as it loses information about the type of dispersion. For instance, when there is little variation in one dimension, generalised variance may vanish, even though the remaining directions exhibit high variability.

\begin{prop}
\label{prop:gv_bound}
Let $\Sigma$ denote the sample covariance for hidden states across layers. The generalised variance $\det \Sigma$ is bounded by the radial variation $\operatorname{Var}(\|\bm{h}\|)$ and (weighted) resultant length $r_w$ as follows:

\begin{equation}
    \log \det \Sigma  \leq d \log \left( \frac{\operatorname{Var}(\|\bm{h}\|) + \mu_{\|\bm{h}\|}^2 (1-r_w^2)}{d} \right),
\end{equation}
where $\mu_{\|\bm{h}\|}=\frac{1}{L+1}\sum_l \|\bm{h}^l\|$ and $r_w=\|\sum_l \bm{h}^l\|/ \sum_l \|\bm{h}^l\| = \|\bar{\bm{h}}\|/\mu_{\|\bm{h}\|}$, with equality if and only if the distribution is isotropic, \textit{i.e.} $\Sigma=\lambda I$.
\end{prop}

\begin{proof}

Let $\lambda_1, \ldots, \lambda_d$ be the eigenvalues of $\Sigma$. Then, $\log\det\Sigma=\sum_{i=1}^{d} \log \lambda_i$. Since logarithm is concave, we apply Jensen's inequality:

\begin{equation}
\label{eq:det}
\begin{split}
    \frac{1}{d} \sum_{i=1}^{d} \log \lambda_i &\leq \log \left( \frac{1}{d} \sum_{i=1}^{d}\lambda_i\right) \\
    &= \log \left( \frac{\text{tr}(\Sigma)}{d}\right)
\end{split}
\end{equation}

Next, we decompose the total variation or trace:
\begin{equation}
\begin{split}
    \text{tr}(\Sigma) =& \frac{1}{L+1}\sum_{l=0}^{L} \|\bm{h}^l - \bar{\bm{h}} \|^2 \\
    =& \left(\frac{1}{L+1} \sum_{l=0}^{L} \|\bm{h}^{l}\|^2 \right) - \|\bar{\bm{h}}\|^2
\end{split}
\end{equation}

Writing $\bm{h}^{l}$ as $ \|\bm{h}^{l}\| \hat{\bm{h}^{l}}$, the first term is the second moment of radius, $\mu_{\|\bm{h}\|^2} = \frac{1}{L+1} \sum_{l=0}^{L} \|\bm{h}^{l}\|^2 = \operatorname{Var}(\|\bm{h}\|) + \mu_{\|\bm{h}\|}^2$. The norm of mean vector can be expressed as $\|\bar{\bm{h}}\| = \mu_{\|\bm{h}\|} r_w$. Thus,

\begin{equation}
\label{eq:trace}
    \text{tr}(\Sigma) = \operatorname{Var}(\|\bm{h}\|) + \mu_{\|\bm{h}\|}^2(1-r_w^2)
\end{equation}

Substituting \Cref{eq:trace} into \Cref{eq:det} yields the desired inequality.

\end{proof}

\subsection{Circular Variance and Pairwise Similarity}
\label{subsec:circ_var}

Recall that circular variance at token $t$ is defined as

\begin{equation}
    c_t = 1 - r_t = 1 - \left\| \frac{1}{L+1} \sum_{l=0}^{L} \hat{\bm{h}}_t^{l} \right\|
\end{equation}

Let the mean pairwise cosine similarity be $S_t = \frac{1}{L(L+1)} \sum_{l\neq m} \hat{\bm{h}}_t^{l} \cdot \hat{\bm{h}}_t^{m}$. The hidden states have been normalised such that $\|\hat{\bm{h}}_t^{l}\|=1 \;\forall l$.

\begin{equation}
\begin{split}
    r_t^2 &= \left\|\frac{1}{L+1} \sum_{l=0}^{L} \hat{\bm{h}}_t^{l} \right\|^2 = \frac{1}{(L+1)^2} \sum_{l,m} \hat{\bm{h}}_t^{l} \cdot \hat{\bm{h}}_{t}^{m} \\
    & = \frac{1}{(L+1)^2} \left(\sum_{l=m} \hat{\bm{h}}_t^{l} \cdot \hat{\bm{h}}_{t}^{m} + \sum_{l\neq m} \hat{\bm{h}}_t^{l} \cdot \hat{\bm{h}}_{t}^{m}\right) \\
    & = \frac{1 + LS_t}{L+1}
\end{split}
\end{equation}

\noindent Expressing $c_t$ in terms of $S_t$ yields

\begin{equation}
    c_t = 1 - \sqrt{\frac{1+LS_t}{L+1}}
\end{equation}

Circular variance accounts for the relationship between all pairs at different layers, instead of only subsequent pairs under CoE.

\section{Computational Cost}
\label{sec:compute}

SIVR is far more efficient than existing sampling-based approaches as the additional computation to extract internal variance features is manageable. In particular, the determinant of the covariance is obtained without performing a large $d^2$ eigendecomposition. Instead, we compute the determinant from a small $(L+1)\times (L+1)$ Gram matrix using Sylvester's theorem or equivalently Singular Value Decomposition (SVD). The other terms, circular variance and predictive entropy, are efficient to compute.

\begin{table}[h]
    \centering
    \small
    \begin{tabular}{l c c}
    \toprule
         \textbf{Method} & \makecell{\textbf{Time (s)} \\ TriviaQA / SciQ} & \makecell{\textbf{AUC} \\ TriviaQA / SciQ} \\
         \midrule
         SE & 5.20 / 7.08 & 84.72 / 73.99 \\
         SAR & 3.51 / 4.17 & 84.45 / 79.44 \\
         SAPLMA & 0.25 / 0.35 & 77.94 / 83.62 \\
         \midrule
         SIVR & 0.26 / 0.37 & 89.31 / 85.70 \\
    \bottomrule
    \end{tabular}
    \caption{Computational time of SIVR compared to sampling-based methods, SE and SAR, and SAPLMA.}
    \label{tab:compute}
\end{table}

In \Cref{tab:compute}, we report the computational time with Llama-3.2-3B-Instruct on Nvidia RTX-6000 with 48 GB of GPU memory. We compare with two sampling-based methods, Semantic Entropy (SE) \citep{kuhn2023semantic} and SAR \citep{duan-etal-2024-shifting} implemented via LM-Polygraph \citep{fadeeva-etal-2024-fact}, and SAPLMA. SIVR is lightweight in computation, adding modest overhead to SAPLMA, and achieves better performance.

\section{Experimental Details}
\label{sec:expt_detail}

\subsection{Dataset}

We evaluate on twelve datasets:
\begin{itemize}[leftmargin=*]
    \item Animals and Facts~\citep{azaria2023internal}: we use all 1008 and 613 statements.
    \item TriviaQA~\citep{joshi2017triviaqa}, MedMCQA~\citep{pal2022medmcqa}: we randomly sample 1000 questions from the validation set.
    \item SciQ~\citep{welbl2017crowdsourcing}, MMLU~\citep{hendrycks2020measuring}: the full validation and test set of 1531 and 1000 questions, respectively, are used.
    \item MGSM~\citep{shi2022language}: we select queries from four languages, English, Bengali, Japanese, and Thai, each with 250 instances.
    \item Counterfact~\citep{meng2022locating}, Common Claims~\citep{casper2023explore}, FEVER~\citep{thorne2018fever}, MATH~\citep{hendrycks2021measuring}, CommonsenseQA~\citep{talmor2018commonsenseqa}: we randomly select 1000 instances. For FEVER, only those with labels ``SUPPORTS'' or ``REFUTES'' are chosen.
\end{itemize}

For probe-based methods, the datasets are split 80-20 for training and evaluation. Following \citet{duan-etal-2024-shifting}, for TriviaQA and SciQ, ROUGE-L \citep{lin2004rouge} is adopted as the correctness metric, with a threshold of 0.7. Table~\ref{tab:prompts} provides the prompts used.

\begin{table}[htbp]
\renewcommand\arraystretch{1.1}
    \centering
    \begin{tabular}{p{0.99\columnwidth}}
         \toprule
         \rowcolor{gray!20}
         {\bf Counterfact, Claims, Animals, Facts, FEVER} \\
         Determine whether the given statement is TRUE or FALSE. The Final Answer must only be either `TRUE.' or `FALSE.' only.
         \texttt{\{statement\}} \\
         Let's think step-by-step. \\

         \midrule
         \rowcolor{gray!20}
         {\bf TriviaQA, SciQ} \\
         Be concise, and output only the final answer.
         \texttt{\{question\}} \\

         \midrule
         \rowcolor{gray!20}
         {\bf MedMCQA} \\
         Answer the following multiple choice medical question. The last line of your response should be of the following format: `Answer: \$LETTER' (without quotes) where LETTER is one of ABCD. Think step by step and output the reasoning process before answering.
         \texttt{\{question\}} \\
         
         \midrule
         \rowcolor{gray!20}
         {\bf MGSM} \\
         Solve this math problem. Give the reasoning steps before giving the final answer on the last line by itself in the format of ``Answer:''. Do not add anything other than the integer answer after ``Answer:''
         \texttt{\{question\}} \\
         
         \midrule
         \rowcolor{gray!20}
         {\bf MATH} \\
         \texttt{\{question\}}
         Please reason step by step, and put your final answer within \verb|\\|boxed{{}} \verb|\\| \\

         \midrule
         \rowcolor{gray!20}
         {\bf MMLU} \\
         Answer the following multiple choice question. The last line of your response should be of the following format: `Answer: \$LETTER' (without quotes) where LETTER is one of ABCD. Think step by step before answering.
         \texttt{\{question\}} \\

         \midrule
         \rowcolor{gray!20}
         {\bf CommonsenseQA} \\
         Answer the following multiple choice common-sense reasoning question. The last line of your response should be of the following format: `Answer: \$LETTER' (without quotes) where LETTER is one of ABCDE. Think step by step and output the reasoning process before answering.
         \texttt{\{question\}} \\
         \bottomrule
    \end{tabular}
    \caption{Prompts for each dataset.}
    \label{tab:prompts}
\end{table}

\subsection{Classifier Details}

A simple transformer encoder model is used for the sequence classification task. The architecture is as shown in \Cref{fig:transformer_arch}. Adam optimiser is used with learning rate $10^{-4}$ and weight decay $10^{-5}$.

\begin{figure}
    \centering
    \includegraphics[width=\linewidth]{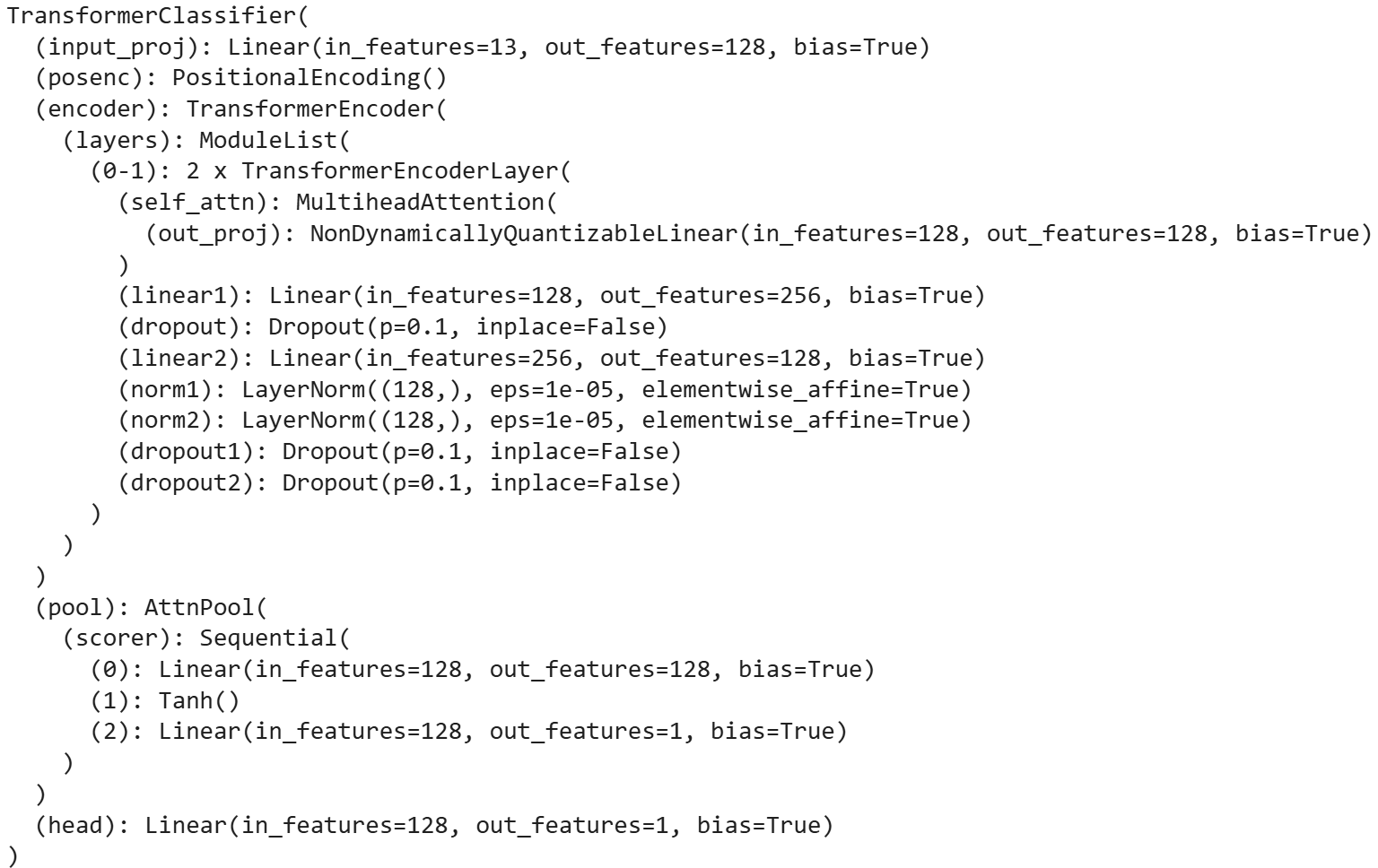}
    \caption{Model architecture of classifier.}
    \label{fig:transformer_arch}
\end{figure}

\section{Additional Experimental Results}
\label{sec:extra_expt}

\subsection{Qwen-3 Experiments}
\label{subsec:qwen_result}

\begin{table*}[t!] 
\renewcommand{\arraystretch}{1.35}
\centering
\setlength{\tabcolsep}{5pt}

\begin{subtable}{\textwidth}
\resizebox{\textwidth}{!}{
\begin{tabular}{l||ccc|ccc|ccc|ccc|ccc|ccc|ccc|ccc}
\toprule
\multirow{2}{*}{\textbf{Method}} & \multicolumn{3}{c|}{\textbf{TriviaQA}}& \multicolumn{3}{c|}{\textbf{SciQ}} & \multicolumn{3}{c|}{\textbf{MedMCQA}} & \multicolumn{3}{c|}{\textbf{MGSM}} & \multicolumn{3}{c|}{\textbf{MATH}} & \multicolumn{3}{c|}{\textbf{MMLU}} &  \multicolumn{3}{c|}{\textbf{CommonsenseQA}} & \multicolumn{3}{c}{\textbf{Average}} \\ \cline{2-25}
    & \textbf{AUC} & \textbf{F95} & \textbf{AUP} & \textbf{AUC} & \textbf{F95} & \textbf{AUP} & \textbf{AUC} & \textbf{F95} & \textbf{AUP} & \textbf{AUC} & \textbf{F95} & \textbf{AUP} & \textbf{AUC} & \textbf{F95} & \textbf{AUP} & \textbf{AUC} & \textbf{F95} & \textbf{AUP} & \textbf{AUC} & \textbf{F95} & \textbf{AUP} & \textbf{AUC} & \textbf{F95} & \textbf{AUP} \\
\bottomrule
\end{tabular}
}
\caption{Qwen-3-4B-Instruct}
\end{subtable}

\begin{subtable}{\textwidth}
\resizebox{\textwidth}{!}{
\begin{tabular}{l||ccc|ccc|ccc|ccc|ccc|ccc|ccc|ccc}
\toprule
\multirow{2}{*}{\textbf{Method}} & \multicolumn{3}{c|}{\textbf{TriviaQA}}& \multicolumn{3}{c|}{\textbf{SciQ}} & \multicolumn{3}{c|}{\textbf{MedMCQA}} & \multicolumn{3}{c|}{\textbf{MGSM}} & \multicolumn{3}{c|}{\textbf{MATH}} & \multicolumn{3}{c|}{\textbf{MMLU}} &  \multicolumn{3}{c|}{\textbf{CommonsenseQA}} & \multicolumn{3}{c}{\textbf{Average}} \\ \cline{2-25}
    & \textbf{AUC} & \textbf{F95} & \textbf{AUP} & \textbf{AUC} & \textbf{F95} & \textbf{AUP} & \textbf{AUC} & \textbf{F95} & \textbf{AUP} & \textbf{AUC} & \textbf{F95} & \textbf{AUP} & \textbf{AUC} & \textbf{F95} & \textbf{AUP} & \textbf{AUC} & \textbf{F95} & \textbf{AUP} & \textbf{AUC} & \textbf{F95} & \textbf{AUP} & \textbf{AUC} & \textbf{F95} & \textbf{AUP} \\
\bottomrule
\end{tabular}
}
\caption{Qwen-3-14B}
\end{subtable}

\caption{Qwen-3 performance.}
\label{tab:qwen_result}
\end{table*}

\begin{table}[htbp]
    \centering
    \resizebox{\linewidth}{!}{
    \begin{tabular}{l| c c c c}
         \toprule
         \textbf{Dataset} & \textbf{Short} & \textbf{Medium} & \textbf{Long} & \textbf{Overall} \\
         \midrule
         TriviaQA & 84.74 & 83.19 & 99.43 & 89.31 \\
         SciQ & 83.19 & 78.89 & 94.54 & 85.70 \\
         \bottomrule
    \end{tabular}
    }
    \caption{AUC with Llama-3.2-3B-Instruct for varying input lengths.}
    \label{tab:input_len}
\end{table}

\begin{table}[htbp]
\centering

\begin{subtable}[t]{\linewidth}
    \centering
    \setlength{\tabcolsep}{1pt}
    \renewcommand{\arraystretch}{1.25}
    \resizebox{\linewidth}{!}{
    \begin{tabular}{l| c c c c c c c | c | c}
         \toprule
         \textbf{Test On} & Trivia & SciQ & Med & MGSM & MATH & MMLU & CSQA & Avg & Test FPR@95 \\
         \midrule
         \textbf{Lookback Ratios} & 4.87 & 1.3 & 3.18 & 21.4 & 6.2 & -0.55 & 6.16 & 6.08 & 91.76 \\
         \textbf{Hidden States} & 36.22 & 17.45 & 8.34 & 34.6 & 20.79 & 4.87 & 10.16 & 18.92 & 92.26 \\
         \textbf{Internal Variance} & 29.33 & 29.45 & -6.29 & 27.48 & 14.09 & 3.66 & 2.33 & 14.29 & 85.33 \\
         \textbf{SIVR} & 32.05 & 23.84 & -3.62 & 32.82 & 18.39 & 4.27 & 8.33 & 16.58 & 87.35 \\
         \bottomrule
    \end{tabular}
    }
    \caption{Average change in FPR@95.}
\end{subtable}

\begin{subtable}[t]{\linewidth}
    \centering
    \setlength{\tabcolsep}{1pt}
    \renewcommand{\arraystretch}{1.25}
    \resizebox{\linewidth}{!}{
    \begin{tabular}{l| c c c c c c c | c | c}
         \toprule
         \textbf{Test On} & Trivia & SciQ & Med & MGSM & MATH & MMLU & CSQA & Avg & Test AUPR \\
         \midrule
         \textbf{Lookback Ratios} & -22.64 & -22.15 & -7.93 & 1.87 & -19.27 & -3.01 & -8.93 & -11.72 & 53.63 \\
         \textbf{Hidden States} & -40.2 & -25.22 & -16.12 & -28.23 & -25.94 & -4.0 & -7.67 & -21.06 & 50.16 \\
         \textbf{Internal Variance} & -9.76 & -12.38 & -9.67 & -22.57 & -16.37 & -1.57 & -6.11 & -11.21 & 59.94 \\
         \textbf{SIVR} & -20.19 & -12.72 & -4.5 & -30.65 & -25.22 & -4.44 & -11.07 & -15.54 & 57.07 \\
         \bottomrule
    \end{tabular}
    }
    \caption{Average change in AUPR.}
\end{subtable}

\label{tab:ood_more}
\caption{Summary of performance with OOD data.}

\end{table}

\begin{table*}[htbp] 
\renewcommand{\arraystretch}{1.35}
\centering
\setlength{\tabcolsep}{1pt}

\begin{subtable}{\textwidth}
\resizebox{\textwidth}{!}{
\begin{tabular}{l||ccc|ccc|ccc|ccc|ccc|ccc|ccc|ccc|ccc|ccc}
\toprule
\multirow{2}{*}{\textbf{Variant}} & \multicolumn{3}{c|}{\textbf{True-False}} & \multicolumn{3}{c|}{\textbf{Fever}} & \multicolumn{3}{c|}{\textbf{TriviaQA}}& \multicolumn{3}{c|}{\textbf{SciQ}} & \multicolumn{3}{c|}{\textbf{MedMCQA}} & \multicolumn{3}{c|}{\textbf{GSM}} & \multicolumn{3}{c|}{\textbf{Math}} & \multicolumn{3}{c|}{\textbf{MMLU}} &  \multicolumn{3}{c|}{\textbf{CommonsenseQA}} & \multicolumn{3}{c}{\textbf{Average}} \\ \cline{2-31}
    & \textbf{AUC} & \textbf{F95} & \textbf{AUP} & \textbf{AUC} & \textbf{F95} & \textbf{AUP} & \textbf{AUC} & \textbf{F95} & \textbf{AUP} & \textbf{AUC} & \textbf{F95} & \textbf{AUP} & \textbf{AUC} & \textbf{F95} & \textbf{AUP} & \textbf{AUC} & \textbf{F95} & \textbf{AUP} & \textbf{AUC} & \textbf{F95} & \textbf{AUP} & \textbf{AUC} & \textbf{F95} & \textbf{AUP} & \textbf{AUC} & \textbf{F95} & \textbf{AUP} & \textbf{AUC} & \textbf{F95} & \textbf{AUP} \\
\bottomrule
\end{tabular}
}

\caption{Llama-3.2-3B-Instruct}
\end{subtable}

\begin{subtable}{\textwidth}
\resizebox{\textwidth}{!}{
\begin{tabular}{l||ccc|ccc|ccc|ccc|ccc|ccc|ccc|ccc|ccc|ccc}
\toprule
\multirow{2}{*}{\textbf{Variant}} & \multicolumn{3}{c|}{\textbf{True-False}} & \multicolumn{3}{c|}{\textbf{Fever}} & \multicolumn{3}{c|}{\textbf{TriviaQA}}& \multicolumn{3}{c|}{\textbf{SciQ}} & \multicolumn{3}{c|}{\textbf{MedMCQA}} & \multicolumn{3}{c|}{\textbf{GSM}} & \multicolumn{3}{c|}{\textbf{Math}} & \multicolumn{3}{c|}{\textbf{MMLU}} &  \multicolumn{3}{c|}{\textbf{CommonsenseQA}} & \multicolumn{3}{c}{\textbf{Average}} \\ \cline{2-31}
    & \textbf{AUC} & \textbf{F95} & \textbf{AUP} & \textbf{AUC} & \textbf{F95} & \textbf{AUP} & \textbf{AUC} & \textbf{F95} & \textbf{AUP} & \textbf{AUC} & \textbf{F95} & \textbf{AUP} & \textbf{AUC} & \textbf{F95} & \textbf{AUP} & \textbf{AUC} & \textbf{F95} & \textbf{AUP} & \textbf{AUC} & \textbf{F95} & \textbf{AUP} & \textbf{AUC} & \textbf{F95} & \textbf{AUP} & \textbf{AUC} & \textbf{F95} & \textbf{AUP} & \textbf{AUC} & \textbf{F95} & \textbf{AUP} \\
\bottomrule
\end{tabular}
}

\caption{Llama-3.1-8B-Instruct}
\end{subtable}

\caption{Ablation with Llama models.}
\label{tab:ablation_meta}
\end{table*}

\begin{table}[htbp] 
\renewcommand{\arraystretch}{1.2}
\centering
\setlength{\tabcolsep}{0.9pt}

\begin{subtable}{0.5\textwidth}
\resizebox{\textwidth}{!}{
\begin{tabular}{l||ccc|ccc|ccc|ccc|ccc}
\toprule
\multirow{2}{*}{\textbf{Method}} & \multicolumn{3}{c|}{\textbf{Counterfact}} & \multicolumn{3}{c|}{\textbf{Common Claim}} & \multicolumn{3}{c|}{\textbf{Animals}}& \multicolumn{3}{c|}{\textbf{Facts}} & \multicolumn{3}{c}{\textbf{FEVER}} \\ \cline{2-16}
    & \textbf{AUC} & \textbf{F95} & \textbf{AUP} & \textbf{AUC} & \textbf{F95} & \textbf{AUP} & \textbf{AUC} & \textbf{F95} & \textbf{AUP} & \textbf{AUC} & \textbf{F95} & \textbf{AUP} & \textbf{AUC} & \textbf{F95} & \textbf{AUP} \\
\midrule
P(True)  &  \large\cellcolor[rgb]{0.99123,0.86313,0.78777} 49.85  &  \large\cellcolor[rgb]{0.97924,0.91911,0.88374} 78.36  &  \large\cellcolor[rgb]{0.97555,0.75917,0.65306} 33.85  &  \large\cellcolor[rgb]{0.94817,0.95894,0.96486} 71.00  &  \large\cellcolor[rgb]{0.83714,0.90634,0.94441} \underline{47.56}  &  \large\cellcolor[rgb]{0.98662,0.88466,0.82468} 61.45  &  \large\cellcolor[rgb]{0.99216,0.85882,0.78039} 54.09  &  \large\cellcolor[rgb]{0.97739,0.92772,0.8985} 72.55  &  \large\cellcolor[rgb]{0.97555,0.75917,0.65306} 24.31  &  \large\cellcolor[rgb]{0.97416,0.75087,0.64245} 60.07  &  \large\cellcolor[rgb]{0.99216,0.85882,0.78039} 67.19  &  \large\cellcolor[rgb]{0.97416,0.75087,0.64245} 19.12  &  \large\cellcolor[rgb]{0.96571,0.96724,0.96809} 66.09  &  \large\cellcolor[rgb]{0.9737,0.94494,0.92803} 70.04  &  \large\cellcolor[rgb]{0.98754,0.88035,0.8173} 38.86 \\
\bottomrule
\end{tabular}
}
\caption{Llama-3.2-3B-Instruct}
\end{subtable}

\begin{subtable}{0.5\textwidth}
\resizebox{\textwidth}{!}{
\begin{tabular}{l||ccc|ccc|ccc|ccc|ccc}
\toprule
\multirow{2}{*}{\textbf{Method}} & \multicolumn{3}{c|}{\textbf{Counterfact}} & \multicolumn{3}{c|}{\textbf{Common Claim}} & \multicolumn{3}{c|}{\textbf{Animals}}& \multicolumn{3}{c|}{\textbf{Facts}} & \multicolumn{3}{c}{\textbf{FEVER}} \\ \cline{2-16}
    & \textbf{AUC} & \textbf{F95} & \textbf{AUP} & \textbf{AUC} & \textbf{F95} & \textbf{AUP} & \textbf{AUC} & \textbf{F95} & \textbf{AUP} & \textbf{AUC} & \textbf{F95} & \textbf{AUP} & \textbf{AUC} & \textbf{F95} & \textbf{AUP} \\
\midrule
P(True)  &  \large\cellcolor[rgb]{0.95402,0.96171,0.96594} 65.48  &  \large\cellcolor[rgb]{0.85467,0.91465,0.94764} 65.83  &  \large\cellcolor[rgb]{0.98662,0.88466,0.82468} 33.15  &  \large\cellcolor[rgb]{0.8313,0.90358,0.94333} 76.18  &  \large\cellcolor[rgb]{0.91895,0.9451,0.95948} 63.01  &  \large\cellcolor[rgb]{0.91895,0.9451,0.95948} 59.62  &  \large\cellcolor[rgb]{0.99216,0.85882,0.78039} 48.34  &  \large\cellcolor[rgb]{0.98662,0.82561,0.73795} 94.13  &  \large\cellcolor[rgb]{0.98247,0.80069,0.70611} 17.66  &  \large\cellcolor[rgb]{0.98293,0.90188,0.85421} 48.44  &  \large\cellcolor[rgb]{0.99216,0.85882,0.78039} 86.49  &  \large\cellcolor[rgb]{0.97693,0.76747,0.66367} 4.03  &  \large\cellcolor[rgb]{0.97278,0.94925,0.93541} 65.87  &  \large\cellcolor[rgb]{0.97739,0.92772,0.8985} 80.76  &  \large\cellcolor[rgb]{0.98524,0.8173,0.72734} 23.47 \\
\bottomrule
\end{tabular}
}
\caption{Llama-3.1-8B-Instruct}
\end{subtable}

\begin{subtable}{0.5\textwidth}
\resizebox{\textwidth}{!}{
\begin{tabular}{l||ccc|ccc|ccc|ccc|ccc}
\toprule
\multirow{2}{*}{\textbf{Method}} & \multicolumn{3}{c|}{\textbf{Counterfact}} & \multicolumn{3}{c|}{\textbf{Common Claim}} & \multicolumn{3}{c|}{\textbf{Animals}}& \multicolumn{3}{c|}{\textbf{Facts}} & \multicolumn{3}{c}{\textbf{FEVER}} \\ \cline{2-16}
    & \textbf{AUC} & \textbf{F95} & \textbf{AUP} & \textbf{AUC} & \textbf{F95} & \textbf{AUP} & \textbf{AUC} & \textbf{F95} & \textbf{AUP} & \textbf{AUC} & \textbf{F95} & \textbf{AUP} & \textbf{AUC} & \textbf{F95} & \textbf{AUP} \\
\midrule
P(True)  &  \large\cellcolor[rgb]{0.97832,0.77578,0.67428} 46.65  &  \large\cellcolor[rgb]{0.98247,0.80069,0.70611} 93.38  &  \large\cellcolor[rgb]{0.97416,0.75087,0.64245} 23.88  &  \large\cellcolor[rgb]{0.97416,0.75087,0.64245} 47.99  &  \large\cellcolor[rgb]{0.98662,0.82561,0.73795} 92.15  &  \large\cellcolor[rgb]{0.97416,0.75087,0.64245} 33.54  &  \large\cellcolor[rgb]{0.98108,0.9105,0.86897} 50.14  &  \large\cellcolor[rgb]{0.9797,0.78408,0.68489} 96.83  &  \large\cellcolor[rgb]{0.98385,0.809,0.71672} 15.91  &  \large\cellcolor[rgb]{0.988,0.83391,0.74856} 51.88  &  \large\cellcolor[rgb]{0.9797,0.78408,0.68489} 90.18  &  \large\cellcolor[rgb]{0.97416,0.75087,0.64245} 7.16  &  \large\cellcolor[rgb]{0.97832,0.77578,0.67428} 48.98  &  \large\cellcolor[rgb]{0.97416,0.75087,0.64245} 96.77  &  \large\cellcolor[rgb]{0.97416,0.75087,0.64245} 18.67 \\
\bottomrule
\end{tabular}
}
\caption{Ministral-8B-Instruct}
\end{subtable}

\begin{subtable}{0.5\textwidth}
\resizebox{\textwidth}{!}{
\begin{tabular}{l||ccc|ccc|ccc|ccc|ccc}
\toprule
\multirow{2}{*}{\textbf{Method}} & \multicolumn{3}{c|}{\textbf{Counterfact}} & \multicolumn{3}{c|}{\textbf{Common Claim}} & \multicolumn{3}{c|}{\textbf{Animals}}& \multicolumn{3}{c|}{\textbf{Facts}} & \multicolumn{3}{c}{\textbf{FEVER}} \\ \cline{2-16}
    & \textbf{AUC} & \textbf{F95} & \textbf{AUP} & \textbf{AUC} & \textbf{F95} & \textbf{AUP} & \textbf{AUC} & \textbf{F95} & \textbf{AUP} & \textbf{AUC} & \textbf{F95} & \textbf{AUP} & \textbf{AUC} & \textbf{F95} & \textbf{AUP} \\
\midrule
P(True)  &  \large\cellcolor[rgb]{0.97463,0.94064,0.92065} 60.03  &  \large\cellcolor[rgb]{0.98016,0.9148,0.87636} 82.08  &  \large\cellcolor[rgb]{0.99216,0.85882,0.78039} 24.33  &  \large\cellcolor[rgb]{0.84298,0.90911,0.94548} \underline{75.72}  &  \large\cellcolor[rgb]{0.89558,0.93403,0.95517} 55.21  &  \large\cellcolor[rgb]{0.93064,0.95063,0.96163} 54.36  &  \large\cellcolor[rgb]{0.97693,0.76747,0.66367} 60.95  &  \large\cellcolor[rgb]{0.98847,0.87605,0.80992} 85.09  &  \large\cellcolor[rgb]{0.97416,0.75087,0.64245} 23.66  &  \large\cellcolor[rgb]{0.9737,0.94494,0.92803} 67.24  &  \large\cellcolor[rgb]{0.96909,0.96647,0.96494} 65.19  &  \large\cellcolor[rgb]{0.97693,0.76747,0.66367} 7.21  &  \large\cellcolor[rgb]{0.98385,0.809,0.71672} 57.76  &  \large\cellcolor[rgb]{0.97924,0.91911,0.88374} 78.87  &  \large\cellcolor[rgb]{0.97416,0.75087,0.64245} 22.85 \\
\bottomrule
\end{tabular}
}
\caption{Qwen-3-4B-Instruct}
\end{subtable}

\begin{subtable}{0.5\textwidth}
\resizebox{\textwidth}{!}{
\begin{tabular}{l||ccc|ccc|ccc|ccc|ccc}
\toprule
\multirow{2}{*}{\textbf{Method}} & \multicolumn{3}{c|}{\textbf{Counterfact}} & \multicolumn{3}{c|}{\textbf{Common Claim}} & \multicolumn{3}{c|}{\textbf{Animals}}& \multicolumn{3}{c|}{\textbf{Facts}} & \multicolumn{3}{c}{\textbf{FEVER}} \\ \cline{2-16}
    & \textbf{AUC} & \textbf{F95} & \textbf{AUP} & \textbf{AUC} & \textbf{F95} & \textbf{AUP} & \textbf{AUC} & \textbf{F95} & \textbf{AUP} & \textbf{AUC} & \textbf{F95} & \textbf{AUP} & \textbf{AUC} & \textbf{F95} & \textbf{AUP} \\
\midrule
P(True)  &  \large\cellcolor[rgb]{0.97555,0.93633,0.91326} 51.36  &  \large\cellcolor[rgb]{0.97647,0.93203,0.90588} 75.41  &  \large\cellcolor[rgb]{0.98524,0.8173,0.72734} 19.26  &  \large\cellcolor[rgb]{0.88973,0.93126,0.95409} 66.04  &  \large\cellcolor[rgb]{0.86052,0.91742,0.94871} 69.06  &  \large\cellcolor[rgb]{0.98016,0.9148,0.87636} 41.02  &  \large\cellcolor[rgb]{0.97416,0.75087,0.64245} 43.06  &  \large\cellcolor[rgb]{0.9797,0.78408,0.68489} 96.21  &  \large\cellcolor[rgb]{0.97416,0.75087,0.64245} 12.94  &  \large\cellcolor[rgb]{0.96571,0.96724,0.96809} 68.30  &  \large\cellcolor[rgb]{0.98939,0.84221,0.75917} 86.05  &  \large\cellcolor[rgb]{0.97693,0.76747,0.66367} 5.14  &  \large\cellcolor[rgb]{0.93064,0.95063,0.96163} 64.70  &  \large\cellcolor[rgb]{0.98385,0.89758,0.84683} 78.82  &  \large\cellcolor[rgb]{0.99031,0.86744,0.79516} 20.45 \\
\bottomrule
\end{tabular}
}
\caption{Qwen-3-14B}
\end{subtable}

\caption{Results for individual true-false topics.}
\label{tab:true_false}
\end{table}

\Cref{tab:qwen_result} shows the evaluation with Qwen-3 models of 4B and 14B parameter size. The results are in line with other model families and sizes, with SIVR consistently outperforming its competitors.

\begin{table}[htbp]
\renewcommand\arraystretch{1}
    \centering
    \begin{tabular}{p{0.99\columnwidth}}
         \toprule
         \rowcolor{blue!20}{\bf Question} \\
         Which score is used for wound infection? A: Glasgow coma scale, B: Southampton score, C: Apgar score, D: SIRS score. \\
         \bf Correct Answer: B \\
         \rowcolor{orange!20}{\bf Response} \\
         The Glasgow Coma Scale is used to assess the level of consciousness in patients, not for wound infection. The\colorbox{red!21!white}{ Southampton}\colorbox{red!17!white}{ score} is used for\colorbox{red!20!white}{ assessing} the severity of sepsis, not specifically for wound infection. The \colorbox{green!10!white}{Apgar}\colorbox{green!53!white}{ score} is used to assess the health of newborns at birth,\colorbox{green!19!white}{ not} for wound infection. The Systemic Inflammatory Response Syndrome (SIRS) score is\colorbox{green!16!white}{ used} to identify patients with systemic inflammatory response syndrome, which\colorbox{green!34!white}{ can}\colorbox{green!19!white}{ be} associated with wound infection. Therefore, the\colorbox{red!21!white}{ correct} answer is:
         
         Answer:\colorbox{red!16!white}{ D}\\
         
         \midrule

         \rowcolor{blue!20}{\bf Question} \\
         Without supervision, who is allowed to give dentures?
            A: Denturist,
            B: Hygienist,
            C: School Dental nurse,
            D: Dental therapist. \\
         \bf Correct Answer: A \\
\bottomrule
    \end{tabular}
    \caption{Qualitative examples. Red implies contribution to hallucination risk.}
    \label{tab:example}
\end{table}

\subsection{Fact-Checking Datasets}
\label{subsec:topic}
\Cref{tab:true_false} provides the results for each dataset with explicit true-false groundtruths.

\subsection{More Ablation}

We report ablation for the Meta models in \Cref{tab:ablation_meta}.

\subsection{Number of Principal Components}
\label{subsec:pc_ablation}

We adopt principle component analysis to avoid overfitting to high-dimensional, model-specific hidden states and keep the supervised head lightweight. This practice is consistent with prior works, \textit{e.g.} \citep{vazhentsev2025token}. We analyse the dependence on the number of principal components (PC) in \Cref{tab:pc_ablation}, and find that performance is stable with 10 PCs.

\begin{table}[htbp]
    \centering
    \setlength{\tabcolsep}{1pt}
    \renewcommand{\arraystretch}{1.25}
    \resizebox{\linewidth}{!}{
    \begin{tabular}{l| c c c c c}
         \toprule
         \textbf{\# PC} & 2 & 5 & 10 & 20 & 30 \\
         \midrule
         \textbf{Average AUC} & 78.60 & 80.02 & 79.29 & 79.68 &  79.24 \\
         \textbf{Average FPR@95} & 60.01 & 57.35 & 58.46 & 61.68 & 57.89 \\
         \textbf{Average AUPR} & 73.63 & 74.38 & 74.06 & 74.10 & 72.91 \\
         \bottomrule
    \end{tabular}
    }
    \caption{Performance with varying number of principal components (\#PC).}
    \label{tab:pc_ablation}
\end{table}

\subsection{More OOD Evaluation}
\label{subsec:extra_ood}
\Cref{fig:ood_more}, along with \Cref{tab:ood_more}, present the OOD change in FPR@95 and AUPR, respectively, with the Ministral-8B-Instruct. The results align with the observation with AUC as well, showing that internal variance improves robustness with respect to all three metrics.

\subsection{Input Length}

We provide additional analysis for SIVR with respect to the length of the generated sequence. In the table below, we bin the generated sequences into short, medium, and long. We observe that SIVR performs better on longer sequences. We believe that with longer sequences, patterns that indicate factual inaccuracy may be more pronounced.

\subsection{Token Contribution to Hallucination}
\label{subsec:example}

With the sequential classifier, we can get a fine-grained view of which tokens drive prediction by using attribution methods that assign contribution from the input. Integrated Gradients \citep{sundararajan2017axiomatic} is a simple training-free method. \Cref{tab:example} demonstrates how IG can highlight to the user which tokens contributed strongly to hallucination risk.

\begin{figure*}[htb]
    \begin{subfigure}[t]{\textwidth}
    \centering
        \includegraphics[width=\linewidth]{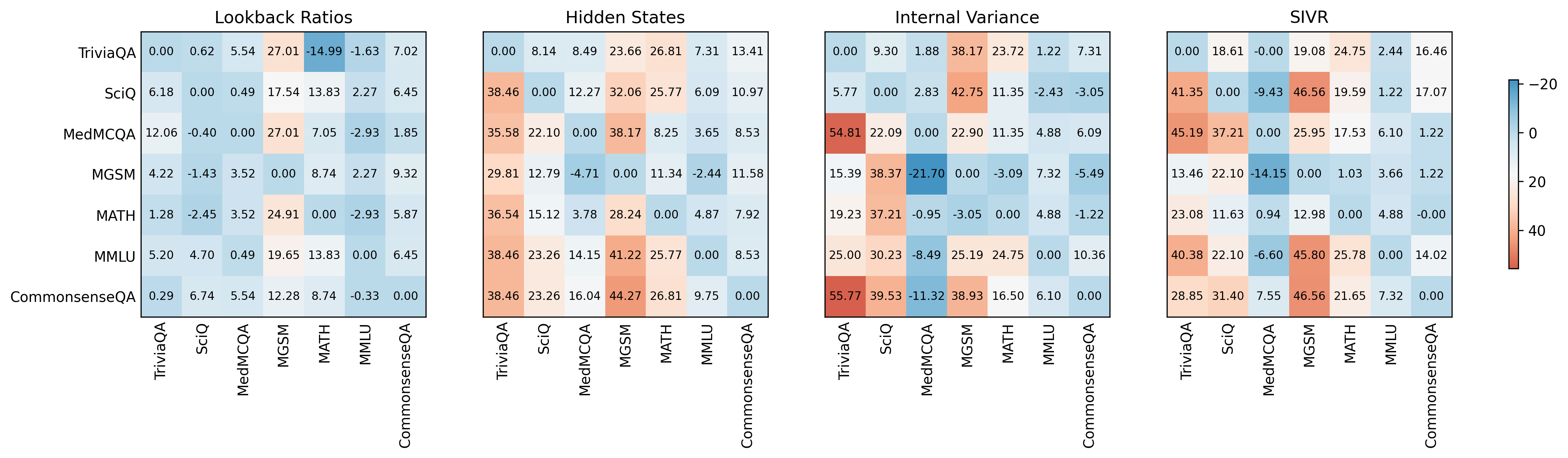}
        \caption{OOD FPR@95 change.}
    \end{subfigure}
    
    \begin{subfigure}[t]{\textwidth}
        \centering
        \includegraphics[width=\linewidth]{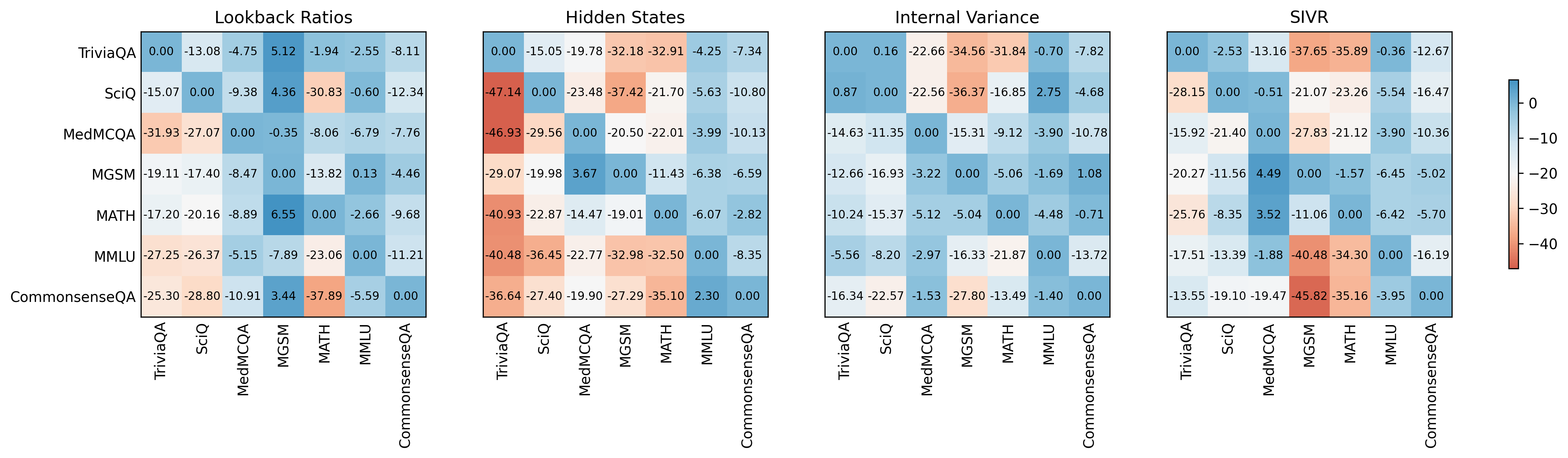}
        \caption{OOD AUPR change.}
    \end{subfigure}
    \caption{More OOD performance with Ministral-8B-Instruct.}
    \label{fig:ood_more}
    \
\end{figure*}

\subsection{Failure Cases}

\Cref{tab:failure} highlights demonstrative failure cases. From these examples, we observe the following: 
\begin{itemize}
    \item False negatives (missed hallucinations): when the error is concentrated in a single decisive token, \textit{e.g.} an important entity, while the rest of the response is fluent and coherent, the dispersion pattern may be weak or not sharply localised, so sequence-level score does not cross hallucination threshold.
    \item False positives (flagging correct answers): when a correct answer contains rare entities or domain-specific terminology, dispersion can be elevated and calibration can be less reliable, leading SIVR to incorrectly flag the response as hallucination.
\end{itemize}
In summary, SIVR can miss localised factual errors, leading to false negatives, and flag rare or technical tokens as incorrect, resulting in false positives.

\begin{table*}[ht]
\renewcommand\arraystretch{1.2}
    \centering
    \resizebox{\linewidth}{!}{
    \begin{tabular}{l|p{3cm}|p{6cm}|c|p{4cm}}
         \toprule
         \multicolumn{1}{c}{Failure Type}  & \multicolumn{1}{c}{Query} & \multicolumn{1}{c}{Response} & \multicolumn{1}{c}{$\hat{p}(y=1)$} & \multicolumn{1}{c}{Comment} \\
         \midrule
         \multirow{2}{*}{False Negative} & Which Gilbert and Sullivan operetta is sub-titled `The Witches Curse'? & The Mikado\colorbox{red!50!white}{</s>} & .251 & Detector wrongly focused on a special token. \\ \cline{2-5}
         & All of the following are the complications in the new born of a diabetic mother except A. Hyperbilirubinemia, B. Hyperglycemia, C. Hypocalcemia, D. Hypomagnesemia & 
         To determine the\colorbox{green!6!white}{ correct} answer,\colorbox{red!100!white}{ let}\colorbox{red!93!white}{'s}\colorbox{red!55!white}{ consider} the\colorbox{green!8!white}{ complications}\colorbox{red!11!white}{ that}\colorbox{green!10!white}{ are}\colorbox{green!5!white}{ typically} associated\colorbox{red!11!white}{ with} newborns of diabetic mothers: ...
        Hyperglycemia: This is a\colorbox{red!9!white}{ common} complication in newborns of\colorbox{red!7!white}{ diabetic} mothers due to the\colorbox{red!7!white}{ mother}'s high blood sugar\colorbox{green!7!white}{ levels}.
        ...
        \colorbox{red!8!white}{Hypomagnesemia}: This is a condition\colorbox{green!7!white}{ where} the newborn\colorbox{green!10!white}{ has} low magnesium\colorbox{green!7!white}{ levels},\colorbox{green!19!white}{ which}\colorbox{green!16!white}{ can}\colorbox{green!16!white}{ also}\colorbox{green!14!white}{ be}\colorbox{green!6!white}{ a} complication in newborn\colorbox{red!5!white}{s} of\colorbox{red!11!white}{ diabetic}\colorbox{red!9!white}{ mothers}\colorbox{red!19!white}{.
        
        }\colorbox{green!32!white}{Given} this information, the complication that is not\colorbox{green!6!white}{ typically}\colorbox{green!8!white}{ associated} with newborns of diabetic mothers is:
        \colorbox{green!13!white}{Answer}:\colorbox{red!6!white}{ D}</s>
         & .467 & 
         Detector did identify the final answer ``D'' as hallucinatory. However, sequence level score is dominated by contributions from other tokens, like semantically less important ones (``let's consider''), wrongly obscuring the predicted probability of hallucination.
         \\ 
         
         \midrule
         \multirow{2}{*}{False Positive} & Under which British monarch was Warren Hastings appointed first Governor of the Presidency of Fort William (Bengal) and the head of the Supreme Council of Bengal, effectively the first Governor-General of India? & Warren Hastings was appointed \colorbox{green!18!white}{by} \colorbox{green!35!white}{King George III}. & .536 & Detector focused on appropriate segment, but does not recognise it as hallucination. \\ \cline{2-5}
         & In female adrenal gland secretes which hormone? A. Progesterone, B. Testosterone, C. Estrogen, D. DHEA & The adrenal glands are \colorbox{red!14!white}{responsible} for producing several hormones, \colorbox{red!14!white}{including} cortisol, aldosterone, \colorbox{red!11!white}{and} androgens. \colorbox{red!40!white}{Among the options given}, \colorbox{red!30!white}{DHEA (Dehydroepiandrosterone)} is an \colorbox{red!20!white}{androgen} produced \colorbox{green!20!white}{by} the \colorbox{green!16!white}{adrenal glands}. \colorbox{green!24!white}{It} is a \colorbox{red!15!white}{precursor} to both testosterone \colorbox{green!12!white}{and} estrogen. Answer: D & .915 &  Detector flags technical terminology as hallucinatory, potentially due to limited training data. \\
         \bottomrule
    \end{tabular}
    }
    \caption{Failure cases.}
    \label{tab:failure}
\end{table*}

\end{document}